%% file: article.tex
\documentclass{article}

\input{preamble.tex}

\graphicspath{{./figures/}}

\usepackage[accepted]{icml2020}

\icmltitlerunning{Extra-gradient with player sampling}

\begin{document}

\twocolumn[
\icmltitle{Extragradient with player sampling for faster Nash equilibrium finding}



\icmlsetsymbol{equal}{*}

\begin{icmlauthorlist}
\icmlauthor{Samy Jelassi}{equal,princeton}
\icmlauthor{Carles Domingo-Enrich}{equal,nyu}
\icmlauthor{Damien Scieur}{samsung}
\icmlauthor{Arthur Mensch}{ens,nyu}
\icmlauthor{Joan Bruna}{nyu}
\end{icmlauthorlist}

\icmlaffiliation{nyu}{NYU CIMS, New York, USA}
\icmlaffiliation{ens}{ENS, DMA, Paris, France}
\icmlaffiliation{princeton}{Princeton University, USA}
\icmlaffiliation{samsung}{Samsung SAIT AI Lab, Montreal, Canada}

\icmlcorrespondingauthor{Samy Jelassi}{sjelassi@princeton.edu}
\graphicspath{{./figures/}}

\icmlkeywords{Convex optimisation, game theory, GANs}

\vskip 0.3in
]

\printAffiliationsAndNotice{\icmlEqualContribution} 

\begin{abstract}
  \input{abstract.tex}
\end{abstract}

  \input{introduction.tex}

  \input{setting.tex}

  \input{theory.tex}

\input{experiments.tex}

  \input{conclusion.tex}


\clearpage

\bibliography{biblio}
\bibliographystyle{icml2020}

\onecolumn

\appendix
\startcontents[sections]

\input{appendices.tex}

\end{document}

%% file: preamble.tex
\usepackage[utf8]{inputenc}
\usepackage[T1]{fontenc}

\usepackage{bm}
\usepackage{dsfont}
\usepackage{csquotes}

\usepackage{xspace}
\usepackage{nicefrac}
\usepackage[super]{nth}
\usepackage{microtype}

\usepackage[english]{babel}

\usepackage{enumitem}

\usepackage{booktabs}
\usepackage{makecell}

\usepackage{graphicx}
\usepackage{subcaption}

\usepackage{xcolor}

\usepackage{algorithm}
\usepackage[noend]{algpseudocode}
\algrenewcommand{\algorithmiccomment}[1]{$\vartriangleright$ #1}
\algrenewcommand{\algorithmicreturn}{\textbf{Return: }}
\algnewcommand\algorithmicinput{\textbf{Input: }}
\algnewcommand\Input{\State \algorithmicinput}

\usepackage{titletoc}

\usepackage[colorlinks=true,bookmarks=true,linkcolor=blue,urlcolor=blue,citecolor=blue,breaklinks=true]{hyperref}

\usepackage{url}

\makeatletter
\addto\extrasenglish{%
\renewcommand{\sectionautorefname}{\S\@gobble}%
\renewcommand{\subsectionautorefname}{\S\@gobble}%
\renewcommand{\subsubsectionautorefname}{\S\@gobble}%
}
\makeatother

\usepackage{changepage}                 
\newlength{\offsetpage}
{\end{adjustwidth}}

\usepackage{amsmath}
\usepackage{amsfonts} 
\usepackage{amssymb}
\usepackage{mathrsfs}
\usepackage{mathtools} 
\mathtoolsset{showonlyrefs}

\usepackage{footnote} 
\makesavenoteenv{algorithm}

\usepackage{amsthm}

\newtheorem{theorem}{Theorem}

\newtheorem{definition}{Definition}
\newtheorem{corollary}{Corollary}
\newtheorem{lemma}{Lemma}
\newtheorem{remark}{Remark}

\newtheorem{assumption}{Assumption}

\newtheorem{subassumptioninner}{Assumption}
\newenvironment{subassumption}[1]{%
  \renewcommand\thesubassumptioninner{#1}%
  \subassumptioninner
}{\endsubassumptioninner}

\input{commands.tex}

%% file: commands.tex
\newcommand{\NN}{\mathbb{N}}

\newcommand{\RR}{\mathbb{R}}

\newcommand{\EE}{\mathbb{E}}

\newcommand{\Aa}{\mathcal{A}}

\newcommand{\Nn}{\mathcal{N}}
\newcommand{\Oo}{\mathcal{O}}
\newcommand{\Pp}{\mathcal{P}}

\renewcommand{\phi}{\varphi}

\newcommand{\Simplex}{\triangle}
\newcommand{\errnash}{\textrm{Err}_N}

\newcommand{\full}{\textnormal{full}}
\newcommand{\cyc}{\textnormal{cyc}}
\newcommand{\rdm}{\textnormal{rand}}

\newcommand{\E}[1]{\mathbb{E}\left[#1\right] }

\newcommand{\foralls}{\forall \,}

\renewcommand{\leq}{\leqslant}
\renewcommand{\geq}{\geqslant}

\renewcommand{\epsilon}{\varepsilon}
\renewcommand{\imath}{\mathrm{i}}

\newcommand{\eqdef}{\triangleq}

\DeclareMathOperator*{\argmin}{argmin}

\DeclareMathOperator*{\interi}{int}

\DeclarePairedDelimiter{\ceil}{\lceil}{\rceil}
\DeclarePairedDelimiter{\floor}{\lfloor}{\rfloor}

\usepackage{accents}

\newlength{\restsubwidth}
\newlength{\restsubheight}
\newlength{\restsubmoreheight}
\newcommand{\rest}[2]{%
        \settowidth{\restsubwidth}{\ensuremath{#2}}
        \settoheight{\restsubheight}{\ensuremath{{}_{#2}}}
        \ensuremath{{#1\hskip 0.5pt}_{\vrule\kern2pt\parbox[b][%
        4pt][b]{\the\restsubwidth}{%
                        \ensuremath{{}_{#2}}}}}
        }

%% file: abstract.tex
Data-driven modeling increasingly requires to find a Nash equilibrium in multi-player games, e.g. when training GANs. In this paper, we analyse a new extra-gradient method for Nash equilibrium finding, that performs gradient extrapolations and updates on a random subset of players at each iteration. This approach provably exhibits a better rate of convergence than full extra-gradient for non-smooth convex games with noisy gradient oracle. We propose an additional variance reduction mechanism to obtain speed-ups in smooth convex games. Our approach makes extrapolation amenable to massive multiplayer settings, and brings empirical speed-ups, in particular when using a heuristic cyclic sampling scheme. Most importantly, it allows to train faster and better GANs and mixtures of GANs.

%% file: introduction.tex

\stepcounter{section}

A growing number of models in machine learning require to optimize over multiple interacting objectives. This is the case of generative adversarial networks \citep{goodfellow2014generative}, imaginative agents \citep{racaniere2017imagination}, hierarchical reinforcement learning \citep{wayne2014hierarchical} and multi-agent reinforcement learning \citep{bu2008comprehensive}. Solving saddle-point problems~\citep[see e.g.,][]{rockafellar_monotone_1970}, that is key in robust learning~\citep{kim_robust_2006} and image reconstruction~\citep{chambolle_firstorder_2011}, also falls in this category. These examples can be cast as games where players are parametrized modules that compete or cooperate
to minimize their own objective functions.

To define a principled solution to a multi-objective optimization problem, we may rely on the notion of Nash equilibrium~\citep{nash1951non}. At a Nash equilibrium, no player can improve its objective by unilaterally changing its strategy.
%
%
The theoretical section of this paper considers the class of \textit{convex $n$-player games}, for which Nash equilibria exist~\citep{rosen_existence_1965}.
Finding a Nash equilibrium in this setting is equivalent to solving a variational inequality problem (VI) with a monotone operator~\citep{rosen_existence_1965,harker1990finite}. This VI can be solved using first-order methods, that are prevalent in single-objective optimization for machine learning. Stochastic gradient descent (the simplest first-order method) is indeed known to converge to local minima under mild conditions met by ML problems~\citep{bottou2008tradeoffs}. Yet, while gradient descent can be applied simultaneously to different objectives, it may fail in finding a Nash equilibrium in very simple settings \citep[see e.g.,][]{sos,gidel2018variational}.
Two alternative modifications of gradient descent are necessary to solve the VI (hence Nash) problem: \textit{averaging}~\citep{magnanti1997averaging,nedic2009subgradient} or \textit{extrapolation} with averaging. The later was introduced as the \textit{extra-gradient} (EG) method by \citet{korpelevich1976extragradient}); it is faster~\citep{nemirovski2004prox} and can handle noisy gradients \citep{solving}. Extrapolation corresponds to an \textit{opponent shaping} step: each player anticipates its opponents' next moves to update its strategy.

\begin{figure*}[t]
  \centering
  \includegraphics[width=\textwidth]{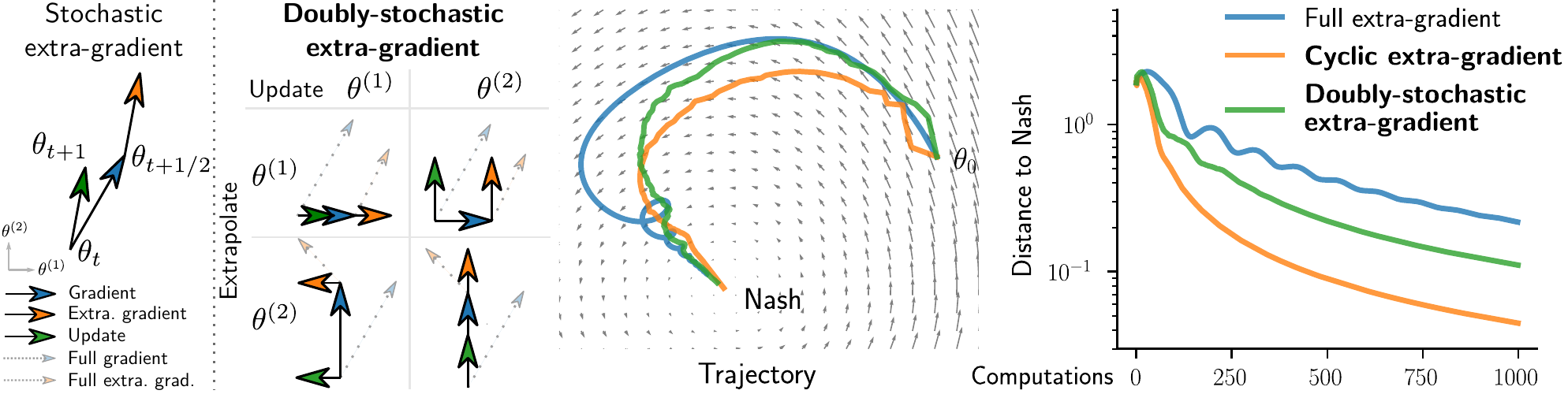}
  \caption{\textit{Left:} We compute masked gradient during the extrapolation and update steps of the extra-gradient algorithm, to perform faster updates. \textit{Right:} Optimization trajectories for doubly stochastic extra-gradient and full-update extra-gradient, on a convex single-parameter two-player convex game. Player sampling improves the expected rate of convergence toward the Nash equilibrium $(0, 0)$.
  }\label{fig:abstract}
  \vspace{-.7em}
\end{figure*}

In $n$-player games, extra-gradient computes $2n$ single player gradients before performing a parameter update. Whether in massive or simple two-players games, this may be an inefficient update strategy: early gradient information, computed at the beginning of each iteration, could be used to perform eager updates or extrapolations, similar to how alternated update of each player would behave. Therefore, we introduce and analyse new extra-gradient algorithms that extrapolate and update random or carefully selected subsets of players at each iteration (\autoref{fig:abstract}).

\begin{itemize}[topsep=0pt,itemsep=0pt,parsep=3pt,leftmargin=15pt]
  \item[--] We review the extra-gradient algorithm for differentiable games and outline its shortcomings (\autoref{subsec:setting}). We propose a doubly-stochastic extra-gradient (DSEG) algorithm (\autoref{subsec:algorithms}) that updates the strategies of a subset of players, thus performing \textit{player sampling}. DSEG performs faster but noisier updates than the original full extra-gradient method (full EG, \cite{solving}), that uses a (once) stochastic gradient oracle. We introduce a variance reduction method to attenuate the noise added by player sampling in smooth games.
  \item[--] We derive convergence rates for DSEG in the convex setting (\autoref{sec:theory}), as summarized in \autoref{tab:results}. Proofs strongly relies on the specific structure of the noise introduced by player sampling. Our rates exhibit a better dependency on gradient noise compared to stochastic extra-gradient, and are thus interesting in the high-noise regime common in machine learning.
  \item[--] Empirically, we first validate that DSEG is faster in massive differentiable convex games with noisy gradient oracles. We further show that non-random player selection improves convergence speed, and provide explanations for this phenomenon. In practical non-convex settings, we find that cyclic player sampling improves the speed and performance of GAN training (CIFAR10, ResNet architecture). The positive effects of extrapolation and alternation combine: DSEG should be used to train GANs, and even more to train \textit{mixtures} of GANs.
\end{itemize}

\begin{table}[bt]
  \caption{
    New and existing \citep{solving} convergence rates for convex games, w.r.t. the number of gradient computations~$k$. Doubly-stochastic extra-gradient (DSEG) \textit{multiplies the noise contribution by a factor $\alpha \triangleq \sqrt{b / n}$}, where $b$ is the number of sampled players among $n$. $G$ bounds the gradient norm. $L$: Lip. constant of losses' gradient. $\sigma^2$ bounds the gradient estimation noise. $\Omega$: diameter of the param. space.}
   \label{tab:results}
  \centering
  \small
  \begin{tabular}{lcc}
    \toprule
    \hspace{-.3cm} $\alpha {\triangleq} \sqrt{b / n}$ & \hspace{-.5cm} Non-smooth  & Smooth \\
    \midrule
    \hspace{-.3cm} \textbf{DSEG} &\hspace{-.7cm}
    $ \Oo\left(n \sqrt{\frac{\Omega}{k}(G^2+\alpha^2\sigma^2)} \right)$ &\hspace{-.3cm} $\Oo\left(\frac{\Omega L n^{3/2}}{\alpha k} + \alpha n \sigma \sqrt{\frac{\Omega}{k}} \right)$
    \\
    \midrule
    \hspace{-.3cm} Full EG &\hspace{-.7cm}
    $ \Oo\left(n \sqrt{\frac{\Omega}{k}(G^2+\sigma^2)} \right)$ & \hspace{-.7cm} $\Oo\left(\frac{\Omega L n^{3/2}}{k} + n \sigma \sqrt{\frac{\Omega}{k}} \right)$ \\
    \bottomrule
  \end{tabular}
\end{table}
 
\section{Related work}\label{subsec:related_work}

\paragraph{Extra-gradient method.}
In this paper, we focus on finding the Nash equilibrium in convex $n$-player games, or equivalently the Variational Inequality problem~\citep{harker1990finite,nemirovski2010accuracy}.
This can be done using extrapolated gradient \citep{korpelevich1976extragradient}, a ``cautious'' gradient descent approach that was promoted by \citet{nemirovski2004prox} and \citet{nesterov2007dual}, under the name \textit{mirror-prox}---we review this work in \autoref{subsec:setting}. \citet{solving} propose a stochastic variant of mirror-prox, that assumes access to a noisy gradient oracle. In the convex setting, their results guarantees the convergence of the algorithm we propose, albeit with very slack rates. Our theoretical analysis refines these rates to show the usefulness of player sampling.
Recently, \citet{universal2019bach} described a smoothness-adaptive variant of this algorithm similar to AdaGrad~\citep{duchi_adaptive_2011}, an approach that can be combined with ours.
\citet{yousefian2018stochastic} consider multi-agent games on networks and analyze a stochastic variant of extra-gradient that consists in randomly extrapolating and updating a single player.
Compared to them, we analyse more general player sampling strategies. Moreover, our analysis holds for non-smooth losses, and provides better rates for smooth losses, through variance reduction. We also analyse precisely the reasons why player sampling is useful (see discussion in \autoref{sec:theory}), an original endeavor.

\paragraph{Extra-gradient in non-convex settings.}

Extra-gradient has been applied in non-convex settings. \citet{optimistic} proves asymptotic convergence results for extra-gradient without averaging in a slightly non-convex case. \citet{gidel2018variational} demonstrate the effectiveness of extra-gradient for GANs. They argue that it allows to escape the potentially chaotic 
behavior of simultaneous gradient updates (examplified by e.g. \citet{cheung_vortices_2019}). Earlier work on GANs propose to replace simultaneous updates with alternated updates, with a comparable improvement \citep{gulrajani_improved_2017}. In \autoref{sec:apps}, we show that alternating player updates while performing opponent extrapolation improves the training speed and quality of GANs.

\paragraph{Opponent shaping and gradient adjustment.}
Extra-gradient can also be understood as an \textit{opponent shaping} method: in the extrapolation step, the player looks one step in the future and anticipates the next moves of his opponents. 
Several recent works proposed algorithms that make use of the opponents' information to converge to an equilibrium \citep{zhanglesser,lola,sos}. In particular, the ``Learning with opponent-learning awareness'' (LOLA) algorithm is known for encouraging cooperation in cooperative games~\citep{lola}. 
Lastly, some recent works proposed algorithms to modify the dynamics of simultaneous gradient descent by adding an adjustment term in order to converge to the Nash equilibrium~\citep{mazumdar2019finding} and avoid oscillations \citep{mechanics,mescheder2017numerics}. One caveat of these works is that they need to estimate the Jacobian of the simultaneous gradient, which may be expensive in large-scale systems or even impossible when dealing with non-smooth losses as we consider in our setting. This is orthogonal to our approach that finds solutions of the original VI problem~\eqref{eq:vip}.

%% file: setting.tex
\section{Solving convex games with partial first-order information}

We review the framework of Cartesian convex games and the extra-gradient method in \autoref{subsec:setting}. Building on these, we propose to augment extra-gradient with player sampling and variance reduction in \autoref{subsec:algorithms}.

\subsection{Solving convex games with gradients}\label{subsec:setting}

In a game, each player observes a loss that depends on the independent parameters of all other players.
\begin{definition}
    A standard $n$-player game is given by a set of $n$ players with parameters $\theta=(\theta^{1},\dots,\theta^{n})\in \Theta\subset\mathbb{R}^d$ where $\Theta$ decomposes into a Cartesian product $\prod_{i=1}^n \Theta^i$. Each player's parameter $\theta^i$ lives in $\Theta^i \subset \RR^{d_i}$. Each player is given a loss function $\ell_i\colon\Theta\rightarrow\mathbb{R}$. 
\end{definition}
For example, generative adversarial network (GAN) training is a standard game between a generator and discriminator that do not share parameters. We make the following assumption over the geometry of losses and constraints, that is the counterpart of the convexity assumption in single-objective optimization.
\begin{assumption}\label{ass:cvx_game}
    The parameter spaces $\Theta_1,\dots, \Theta_n$ are compact, convex and non-empty.
    Each player's loss $\ell_i(\theta^i,\theta^{-i})$ is convex in its parameter $\theta^i$ and concave in $\theta^{-i}$, where $\theta^{-i}$ contains all other players' parameters. Moreover, $\sum_{i=1}^n\ell_i(\theta)$ is convex in $\theta$.
\end{assumption}
\autoref{ass:cvx_game} implies that $\Theta$ has a diameter $\Omega \triangleq \max_{u,z \in \Theta} {\|u-z\|}_2.$ Note that the losses may be non-differentiable. A simple example of Cartesian convex games satisfying~\autoref{ass:cvx_game}, that we will empirically study in \autoref{sec:apps}, are matrix games (e.g., rock-paper-scissors) defined by a positive payoff matrix $A \in \RR^{d \times d}$, with parameters $\theta$ corresponding to $n$ mixed strategies $\theta_i$ lying in the probability simplex $\Simplex^{d_i}$.

\paragraph{Nash equilibria.} Joint solutions to minimizing losses ${(\ell_i)}_i$ are naturally defined as the set of \textit{Nash equilibria}~\citep{nash1951non} of the game. In this setting, we look for equilibria $\theta_\star \in \Theta$ such that
\begin{equation}\label{def:nash}
    \foralls i \in [n], \quad \ell_i(\theta^i_\star,\theta^{-i}_{\star}) = \min_{\theta^i \in \Theta^i} \ell_i(\theta^i,\theta^{-i}_{\star}).
\end{equation}
A Nash equilibrium is a point where no player can benefit by changing his strategy while the other players keep theirs unchanged.  
\autoref{ass:cvx_game} implies the existence of a Nash equilibrium \citep{rosen_existence_1965}. We quantify the inaccuracy of a solution $\theta$ by the \textit{functional Nash error}, also known as the \citet{nikaido_note_1955} function:
\begin{equation}\label{eq:nash_err}
     \errnash(\theta) \triangleq \sum_{i=1}^n \left[ \ell_i(\theta)-\min_{z\in \Theta_i}\ell_i(z,\theta^{-i})\right].
 \end{equation}
This error, computable through convex optimization, quantifies the gain that each player can obtain when deviating alone from the current strategy. In particular, $\errnash(\theta) = 0$ if and only if $\theta$ is a Nash equilibrium; thus $\errnash(\theta)$ constitutes a propose indication of convergence for sequence of iterates seeking a Nash equilibrium. We bound this value in our convergence analysis (see \autoref{sec:theory}). 

\paragraph{First-order methods and extrapolation.}In convex games, as the losses $\ell_i$ are (sub)differentiable, we may solve~\eqref{def:nash} using first-order methods. We assume access to the \textit{simultaneous gradient} of the game
\begin{equation} \label{eq:sim_grad}
    F \triangleq (\nabla_1 \ell_1,\dots,\nabla_n \ell_n)^{\top}\in \mathbb{R}^d,
\end{equation}
where we write $\nabla_i \ell_i \triangleq \nabla_{\theta^i} \ell_i$.
It corresponds to the concatenation of the gradients of each player's loss with respect to its own parameters, and may be noisy. The losses $\ell_i$ may be non-smooth, in which case the gradients $\nabla_i \ell_i$ can be replaced by any subgradients. Simultaneous gradient descent, that explicitly discretizes the flow of the simultaneous gradient may converge slowly---e.g., in matrix games with skew-symmetric payoff and noiseless gradient oracle, convergence of the average iterate demands \textit{decreasing} step-sizes. The extra-gradient method \citep{korpelevich1976extragradient} provides better guarantees \citep{nemirovski2004prox,solving}---e.g., in the previous example, the step-size can remain constant. We build upon this method.
 
Extra-gradient consists in two steps: first, take a gradient step to go to an extrapolated point. Then use the gradient at the extrapolated point to perform a gradient step from the original point: at iteration~$\tau$,
\begin{equation}\label{eq:extra_grad}
    \begin{split}
        \text{(extrapolation)}\quad &\theta_{\tau+1/2}=p_{\Theta}[\theta_{\tau}-\gamma_{\tau}F(\theta_{\tau})],\\ 
        \text{(update) }\quad&\theta_{\tau+1} = p_{\Theta}[\theta_{\tau} -\gamma_{\tau}F(\theta_{\tau+1/2})], 
    \end{split}    
\end{equation}
where $p_{\Theta}[\cdot]$ is the Euclidean projection onto the constraint set $\Theta$, i.e. $p_{\Theta}[z] =\argmin_{\theta \in \Theta} \|\theta-z\|_2^2$. This "cautious" approach allows to escape cycling orbits of the simultaneous gradient flow, that may arise around equilibrium points with skew-symmetric Hessians (see \autoref{fig:abstract}). The generalization of extra-gradient to general Banach spaces equipped by a Bregman divergence was introduced as the \textit{mirror-prox} algorithm~\citep{nemirovski2004prox}. The new convergence results of \autoref{sec:theory} extend to the \textit{mirror} setting (see \autoref{sec:miror_prox}). As recalled in \autoref{tab:results}, \citet{solving} provide rates of convergence for the average iterate $\hat \theta_t = \frac{1}{t} \sum_{\tau = 1}^t \theta_\tau$. Those rates are introduced for the equivalent variational inequality (VI) problem, finding
 \begin{equation}\label{eq:vip}
    \theta_{\star} \in \Theta \;\text{such that}\; F(\theta_{\star})^{\top}(\theta-\theta_{\star})\geq 0\;  \foralls \theta \in \Theta,
\end{equation}
where \autoref{ass:cvx_game} ensures that the simultaneous gradient $F$ is a monotone operator (see~\autoref{sec:link_vi_nash} for a review).

\subsection{DSEG: Partial extrapolation and update for extra-gradient}\label{subsec:algorithms}

The proposed algorithms are theoretically analyzed in the convex setting~\autoref{sec:theory}, and empirically validated in convex and non-convex setting in~\autoref{sec:apps}.

\paragraph{Caveats of extra-gradient.} In systems with large number of players, an extra-gradient step may be computationally expensive due to the high number of backward passes necessary for gradient computations. Namely, at each iteration, we are required to compute $2n$ gradients before performing a first update. This is likely to be inefficient, as we could use the first computed gradients to perform a first extrapolation or update. This remains true for games down to two players. In a different setting, stochastic gradient descent~\citep{robbins_stochastic_1951} updates model parameters before observing the whole data, assuming that partial observation is sufficient for progress in the optimization loop. Similarly, in our setting, partial gradient observation should be sufficient to perform extrapolation and updates toward the Nash equilibrium.
\paragraph{Player sampling.}While standard extra-gradient performs at each iteration two passes of player's gradient computation, we therefore compute \textit{doubly-stochastic simultaneous gradient estimates}, where only the gradients of a random subset of players are evaluated. This corresponds to evaluating a simultaneous gradient that is affected by \textit{two} sources of noise. We sample a mini-batch $\mathcal{P}$ of players of size $b \leq n$, and compute the gradients for this mini-batch only. Furthermore,
we assume that the gradients are noisy estimates, e.g., with noise coming from data sampling. We then compute a doubly-stochastic simultaneous gradient estimate $\tilde{F}$ as $\tilde{F} \triangleq (\tilde{F}^{(1)},\dots,\tilde{F}^{(n)})^{\top}\in \mathbb{R}^d$ where
\begin{equation}\label{eq:noisy_sim_grad}
     \tilde{F}^{(i)}(\theta, \Pp) \triangleq 
    \begin{cases}
        \frac{n}{b} \cdot g_i(\theta) & \text{if } i \in \mathcal{P}\\
        0_{d_i} & \text{otherwise}
    \end{cases},
\end{equation}
and $g_i(\theta)$ is a noisy unbiased estimate of $\nabla_i\ell_i(\theta).$ The factor $n/b$ in~\eqref{eq:noisy_sim_grad} ensures that the doubly-stochastic simultaneous gradient estimate is an unbiased estimator of the simultaneous gradient. Doubly-stochastic extra-gradient (DSEG) replaces the full gradients in the update \eqref{eq:extra_grad} by the oracle \eqref{eq:noisy_sim_grad}, as detailed in~\autoref{alg:doubly_stoch}.

\begin{algorithm}[t]
\caption{Doubly-stochastic extra-gradient.} 
\label{alg:doubly_stoch}
\begin{algorithmic}[1]
\State \textbf{Input}: initial point $\theta_{0}\in \mathbb{R}^{d},$  stepsizes $(\gamma_{\tau})_{\tau \in [t]}$, mini-batch size over the players $b\in [n]$.
\State With variance reduction (VR), $R \gets \tilde{F}(\theta_0,[1,n])$ as in~\eqref{eq:noisy_sim_grad}, i.e. the full simultaneous gradient.

\For {$\tau=0,\dots,t$}
        \State Sample mini-batches of players $\Pp$, $\Pp'$.
        \State Compute $\tilde F_{\tau+\frac{1}{2}} = \tilde{F}(\theta_\tau, \Pp)$ using \eqref{eq:noisy_sim_grad} or VR (\autoref{alg:oracle_grad2}).
        \State Extrapolation step: $\theta_{\tau+\frac{1}{2}} \gets p_{\Theta}[ \theta_{\tau}-\gamma_{\tau}
        \tilde{F}_{\tau+\frac{1}{2}}]$.
        \State Compute
         $\tilde{F}_{\tau+1} = \tilde{F}(\theta_{\tau + \frac{1}{2}}, \Pp')$ using~\eqref{eq:noisy_sim_grad} or VR 
        \State Gradient step: $\theta_{\tau+1} \gets p_{\Theta}[\theta_{\tau}- \gamma_{\tau}\tilde{F}_{\tau+1}]$.

\EndFor
\State \textbf{Return} $\hat{\theta}_t = [ \sum_{\tau=0}^t \gamma_{\tau}]^{-1}\sum_{\tau=0}^t \gamma_{\tau}\theta_{\tau}.$
\end{algorithmic}
\end{algorithm}

\begin{algorithm}[t]
    \caption{Variance reduced estimate of the simultaneous gradient with doubly-stochastic sampling} 
    \label{alg:oracle_grad2}
    \begin{algorithmic}[1]
    
    \State \textbf{Input}: point $\theta\in \mathbb{R}^d$, 
    mini-batch $\Pp$, table of previous gradient estimates $R \in \mathbb{R}^d$.
    \State Compute $\tilde{F}(\theta, \Pp)$ as specified in equation~\eqref{eq:noisy_sim_grad}.
    \For {$i \in \mathcal{P}$}
            \State Compute
            $\bar {F}^{(i)} \gets \tilde{F}^{(i)}(\theta) + (1 - \frac{n}{b}) R^{(i)}$
            \State Update $R^{(i)} \gets \frac{b}{n}\tilde{F}^{(i)}(\theta) = g_i(\theta)$
    \EndFor
    \State For $i \notin \mathcal{P}$, set $\bar {F}^{(i)} \gets R^{(i)}$.
    \State \textbf{Return} estimate $\bar{F} = (\bar{F}^{(1)},...,\bar{F}^{(n)})$, table $R$.
    
    \end{algorithmic}
\end{algorithm}

\paragraph{Variance reduction for player noise.} To obtain faster rates in convex games with smooth losses, we propose to compute a variance-reduced estimate of the gradient oracle \eqref{eq:noisy_sim_grad}. This mitigates the noise due to player sampling. Variance reduction is a technique known to accelerate convergence under smoothness assumptions in similar settings.
While \citet{palaniappan2016stochastic,iusem2017extragradient,reducing2019chavdarova} apply 
variance reduction on the noise coming from the gradient estimates, we apply it to the noise coming from the sampling over the players. We implement 
this idea in \autoref{alg:oracle_grad2}. We keep an estimate of $\nabla_i \ell_i$ for each player in a table $R$, which we use to compute \textit{unbiased} gradient estimates with lower variance, akin to the approach of SAGA~\citep{defazio2014saga} to reduce the variance of data noise.  

\paragraph{Player sampling strategies.} For convergence guarantees to hold, each player must have an equal probability of being sampled (\textit{equiprobable player sampling condition}). Sampling uniformly over $b$-subsets of $[n]$ is a reasonable way to fulfill this condition as all players have probability $p=b/n$ of being chosen.

As a strategy to accelerate convergence, we propose to cycle over the $n(n-1)$ pairs of different players (with $b=1$). At each iteration, we extrapolate the first player of the pair and update the second one. We shuffle the order of pairs once the block has been entirely seen. This scheme bridges extrapolation and alternated gradient descent: for GANs, it corresponds to extrapolate the generator before updating the discriminator, and vice-versa, cyclically. Although its convergence is not guaranteed, cyclic sampling over players is powerful for convex quadratic games (\autoref{sec:quadratic}) and GANs (\autoref{sec:gans}).



%% file: theory.tex

\section{Convergence for convex games}\label{sec:theory}

We derive new rates for DSEG with random player sampling, improving the analysis of~\citet{solving}.
Player sampling can be seen as an extra source of noise in the gradient oracle. Hence the results of \citeauthor{solving} on stochastic extra-gradient guarantees the convergence of DSEG, as we detail in~\autoref{cor:juditsky_play_sampl}. Unfortunately, the convergence rates in this corollary do not predict any improvement of DSEG over full extra-gradient. Our main theoretical contribution is therefore a refinement of these rates for player-sampling noise. Improvements are obtained both for non-smooth and smooth losses, the latter using the proposed variance reduction approach. 
Our results predict better performance for DSEG in the high-noise regime. Results are stated here in Euclidean spaces for simplicity; they are proven in the more general mirror setting in~\autoref{sec:proofs}. In the analysis, we separately consider the two following assumptions on the losses.

\begin{subassumption}{2a}[Non-smoothness] \label{ass:non-smooth}
For each $i\in [n],$ the loss $\ell_i$ 
has a bounded  subgradient, namely
$\max_{h \in \partial_i \ell_i(\theta)} {\|h\|}_2 \leq G_i$ for all $\theta \in \Theta$.
In this case, we also define the quantity $G=\sqrt{\sum_{i=1}^n G_i^2/n}.$
\end{subassumption}

\stepcounter{assumption}

\begin{subassumption}{2b}[Smoothness] \label{ass:smooth} For each $i\in [n],$ the loss $\ell_i$ is once-differentiable and $L$-smooth, i.e.\ 
${\|\nabla_i\ell_i(\theta)-\nabla_i\ell_i(\theta')\|}_2 \leq L{\|\theta-\theta'\|}_2,$
for $\theta,\theta'\in \Theta.$
\end{subassumption}

Similar to \citet{solving, robbins_stochastic_1951}, we assume unbiasedness of the gradient estimate and boundedness of the variance.
\begin{assumption} \label{ass:unbiased_var_bnd}
    For each player $i$, the noisy gradient $g_i$ is unbiased and has bounded variance:
    \begin{equation}\label{eq:unbiased_var_bnd}
        \begin{split}
            \foralls \theta \in \Theta, \quad &\mathbb{E}[g_i(\theta)] = \nabla_i \ell_i(\theta),\\
            &\mathbb{E}[{\|g_i(\theta)-\nabla_i\ell_i(\theta)\|}_2^2] \leq \sigma^2.
        \end{split}
    \end{equation}
\end{assumption}
To compare DSEG to simple stochastic EG, we must take into account the cost of a single iteration, that we assume proportional to the number $b$ of gradients to estimate at each step. We therefore set $k \triangleq 2\,b\,t$ to be the number of gradients estimates computed up to iteration $t$, and re-index the sequence of iterate $(\hat \theta_t)_{t \in \NN}$ as $(\hat \theta_k))_{k \in 2b\NN}$. We give rates with respect to $k$ in the following propositions. 

\subsection[Slack rates]{Slack rates derived from \citeauthor{solving}}

Let us first recall the rates obtained by \citet{solving} with noisy gradients but no player sampling.
 
 \begin{theorem}[Adapted from \citet{solving}] \label{juditsky_thm_grad_comp}
   We consider a convex $n$-player game where \autoref{ass:non-smooth} and \autoref{ass:unbiased_var_bnd} hold. We run \autoref{alg:doubly_stoch} for $t$ iterations without player sampling, thus performing $k = 2\,n\,t$ gradient evaluations. With optimal constant stepsize, the expected Nash error verifies
 \begin{equation} \label{a1juditsky}
     \E{ \errnash(\hat{\theta}_{k}) } \leq 14 n \sqrt{\frac{\Omega}{3k}\left(G^2 + 2\sigma^2 \right)}. 
 \end{equation}
 Assuming smoothness~(\autoref{ass:smooth}) and optimal stepsize,
 \begin{equation} \label{a2juditsky}
     \E{\errnash(\hat{\theta}_{k}) } \leq \max \left\{ \frac{7 \Omega L n^{3/2}}{k}, 14 n \sqrt{\frac{2\Omega \sigma^2}{3k}}\right\}.
 \end{equation}

 \end{theorem}

Player sampling fits within the framework of noisy gradient oracle \eqref{eq:unbiased_var_bnd}, replacing the gradient estimates $(g_i)_{i \in [n]}$ with the
estimates $(\tilde F^{(i)})_{i \in [n]}$ from \eqref{eq:noisy_sim_grad}, and updating the variance $\sigma^2$ accordingly. We thus derive the following corollary.

 \begin{corollary}\label{cor:juditsky_play_sampl}
   We consider a convex $n$-player game where \autoref{ass:non-smooth} and \autoref{ass:unbiased_var_bnd} hold. We run \autoref{alg:doubly_stoch} for $t$ iterations with equiprobable player sampling, thus performing $k = 2\,b\,t$ gradient evaluations. With optimal constant stepsize, the expected Nash error verifies
 \begin{equation}
     \E{ \mathrm{Err}_N(\hat{\theta}_{k}) } \leq \Oo\left(n\sqrt{\frac{\Omega }{k}\left(\frac{n}{b}G^2 + \sigma^2\right)}\right). 
 \end{equation}
 Assuming smoothness~(\autoref{ass:smooth}) and optimal stepsize,
 \begin{equation}\label{a2juditsky_naive}
     \E{ \errnash(\hat{\theta}_{k}) } \leq \Oo\left(\frac{\Omega L n^{3/2}}{k} + n\sqrt{\frac{\Omega}{k} (\frac{n}{b} L^2 \Omega^2 + \sigma^2)}\right).
 \end{equation}
 \end{corollary}
 The proof is in \autoref{sec:rand_noise}. The notation $\Oo(\cdot)$ hides numerical constants. Whether in the smooth or non-smooth case, the upper-bounds from \autoref{cor:juditsky_play_sampl} does not predict any improvement due to player sampling, as the factor before the gradient size $G$ or $L \Omega$ is increased, and the factor before the noise variance $\sigma$ remains constant.

 \subsection{Tighter rates using noise structure}

 Fortunately, a more cautious analysis allows to improve these bounds, by taking into account the noise structure induced by sampling in \eqref{eq:noisy_sim_grad}. We provide a new result in the non-smooth case, proven in \autoref{sec:doub_sto_mirr_prox}.
\begin{theorem}\label{thm:non_smooth}
    We consider a convex $n$-player game where \autoref{ass:non-smooth} and \autoref{ass:unbiased_var_bnd} hold. We run \autoref{alg:doubly_stoch} for $t$ iterations with equiprobable player sampling, thus performing $k = 2\,b\,t$ gradient evaluations. With optimal constant stepsize, the expected Nash error verifies
 \begin{equation} \label{minconst0}
    \begin{split}
        \E{ \errnash(\hat{\theta}_{k}) } \leq \Oo\left(n \sqrt{\frac{\Omega}{k} \left(G^2 + \frac{b}{n}\sigma^2\right)}\right).
    \end{split}
 \end{equation}
 \end{theorem}
Compared to \autoref{cor:juditsky_play_sampl}, we obtain a factor $\sqrt{\frac{b}{n}}$ in front of the noise term $\frac{\sigma}{\sqrt{k}}$, without changing the constant before the gradient size $G$.
 We can thus expect faster convergence with noisy gradients. \eqref{minconst0} is tightest when sampling a single player, i.e. when $b = 1$.
 
A similar improvement can be obtained with smooth losses thanks to the variance reduction technique proposed in \autoref{alg:oracle_grad2}. This is made clear in the following result, proven in \autoref{sec:doub_sto_vr}.

\begin{theorem}\label{thm:VR}
    We consider a convex $n$-player game where \autoref{ass:non-smooth} and \autoref{ass:unbiased_var_bnd} hold. We run \autoref{alg:doubly_stoch} for $t$ iterations with equiprobable player sampling, thus performing $k = 2\,b\,t$ gradient evaluations. \autoref{alg:oracle_grad2} yields gradient estimates. With optimal constant stepsize, the expected Nash error verifies
\begin{equation} \label{vrresult0}
     \E{ \errnash(\hat{\theta}_{k}) } {\leq} \Oo\Big(\sqrt{\frac{n}{b}} \frac{\Omega L n^{3/2}}{k} + \sqrt{\frac{b}{n}} n \sqrt{\frac{\Omega\sigma^2}{k}} \Big).
\end{equation}
\end{theorem}
%

 The upper-bound~\eqref{vrresult0} should be compared 
 with the bound of full extra-gradient~\eqref{a2juditsky}---that it recovers for $b=n$. 
With player sampling, the constant before the gradient size $L \Omega$ 
 is bigger of a factor $\sqrt{\frac{n}{b}}$.
On the other hand, the constant before the noise term $\sigma$ 
is smaller of a factor $\sqrt{\frac{n}{b}}$. Player sampling is therefore 
beneficial when the noise term dominates, which is the case whenever
 the number of iterations is such that
$k \geq \frac{\Omega L^2 n}{\sigma^2} \left(\frac{n}{b}\right)^2$.
For $k \to \infty$, the bound \eqref{vrresult0} is once again tightest by sampling a random single player.

To sum up, doubly-stochastic extra-gradient convergence is controlled with a better rate than stochastic extra-gradient (EG) with non-smooth losses; 
with smooth losses, DSEG exhibits the same rate structure in $\frac{1}{k} + \frac{1}{\sqrt{k}}$ as stochastic EG, 
with a better dependency on the noise but worse dependency on the gradient smoothness. 
In the high noise regime, or equivalently when demanding high precision results, 
DSEG brings the same improvement of a factor $\sqrt{\frac{b}{n}}$ before the constant $\frac{\sigma}{\sqrt{k}}$, for both smooth and non-smooth problems.

\paragraph{Step-sizes.} The stepsizes of the previous propositions are assumed to be constant and are optimized knowing the geometry of the problem. They are explicit in \autoref{sec:proofs}. 
As in full extra-gradient, convergence can be guaranteed without such knowledge using decreasing step-sizes. In experiments, we perform a grid-search over stepsizes to obtain the best results given a computational budget~$k$.

%% file: experiments.tex
\begin{figure*}[t]
    \centering
    \subcaptionbox{5-player games with increasingly hard configurations.\label{fig:quadratic_convergence_5}}{\includegraphics[width=\textwidth]{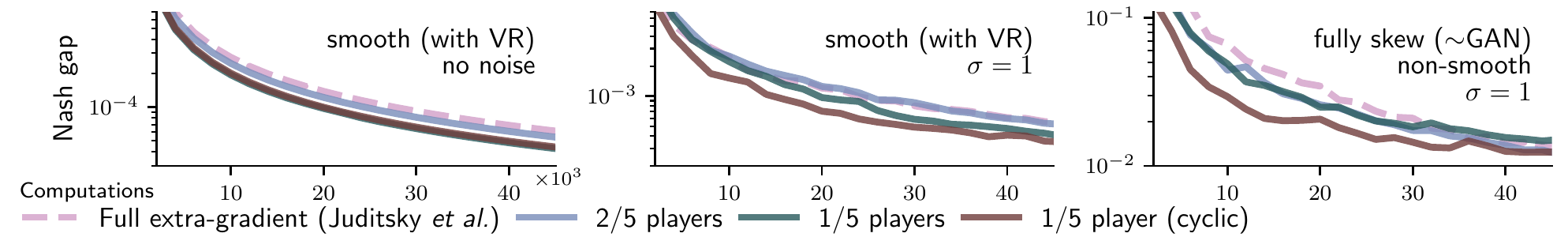}}
    
    \subcaptionbox{50-player games with increasingly hard configurations.\label{fig:quadratic_convergence_50}}{\includegraphics[width=\textwidth]{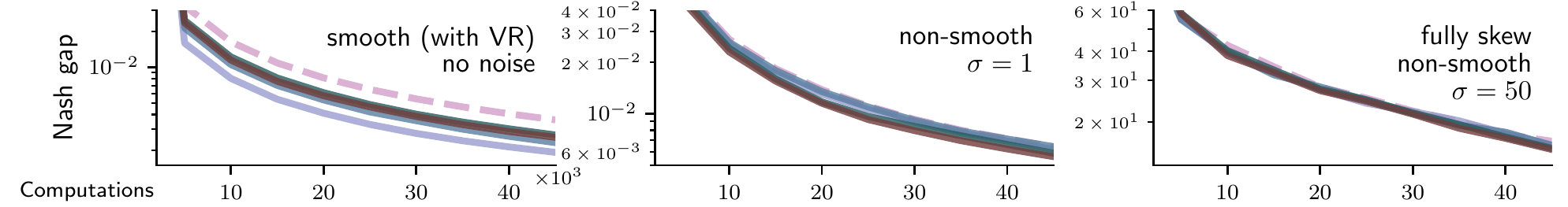}}

    \subcaptionbox{50-player smooth game with increasing noise (sampling with variance reduction).\label{fig:quadratic_convergence_noise}}{\includegraphics[width=\textwidth]{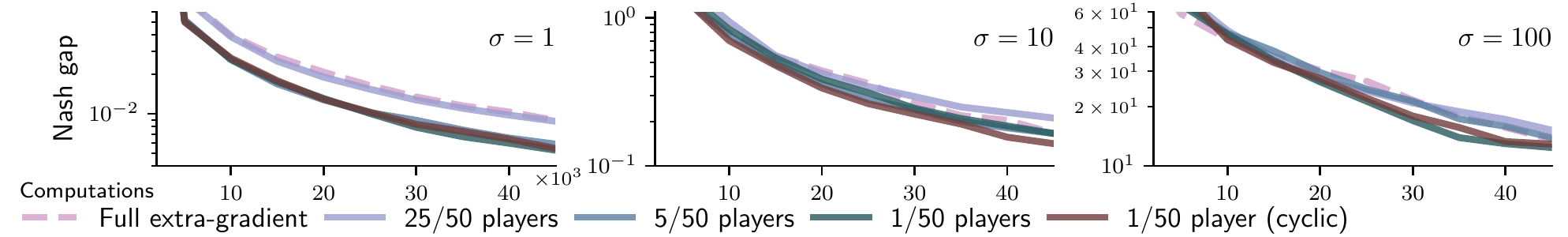}}

    \caption{Player sampled extra-gradient outperform vanilla extra-gradient for small noisy/non-noisy smooth/non-smooth games. Cyclic sampling performs better than random sampling, especially for 5 players (a). Higher sampling ratio is beneficial in high noise regime (c), Curves averaged over 5 games and 5 runs.\vspace{-1em}}\label{fig:quadratic_convergence}
\end{figure*}

\section{Convex and non-convex applications}\label{sec:apps}

We show the performance of doubly-stochastic extra-gradient in the setting of quadratic games, comparing different sampling schemes. We assess the speed and final performance of DSEG in the practical context of GAN training. A \textit{PyTorch}/\textit{Numpy} package is attached.

\subsection{Random convex quadratic games}\label{sec:quadratic}

We consider a game where $n$ players can play $d$ actions, with payoffs provided by a matrix $A \in \RR^{n d \times n d}$, an horizontal stack of matrices $A_i \in \RR^{(d  \times nd)}$ (one for each player). The loss function $\ell_i$ of each player is defined as its expected payoff given the $n$ mixed strategies $(\theta^1,\dots,\theta^n)$, i.e. $ \foralls i \in [n],\; \foralls \theta \in \Theta = \Simplex^{d_1}\times \cdots \times \Simplex^{d_n},$
\begin{equation}
   \ell_i(\theta^{i}, \theta_{-i}) = {\theta^i}^\top A_i \theta + \lambda \Vert \theta^i - \frac{1}{d_i} \Vert_1,
\end{equation}
where $\lambda$ is a regularization parameter that introduces non-smoothness and pushes strategies to snap to the simplex center. 
The positivity of $A$, i.e. $\theta^\top A \theta \geq 0$ for all $\theta \in \Theta$, is equivalent to the convexity of the game.

\begin{figure}[t]
    \includegraphics[width=\linewidth]{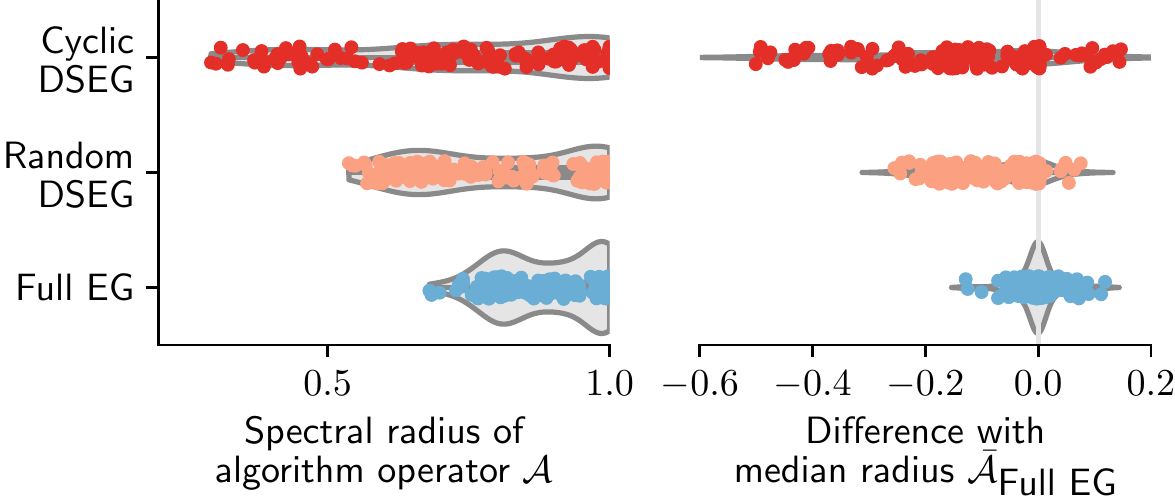}
    \caption{\textit{Left:} Spectral radii of operators for random 2-player matrix games. \textit{Right:} each radius is compared to the median radius obtained for full extra-gradient, within each category of skewness and conditioning of random payoff matrices. Cyclic sampling lowers spectral radii and improve convergence rates.}\label{fig:mean_radius}
\end{figure}

\paragraph{Experiments.}We sample $A$ as the weighted sum of a random symmetric positive definite matrix and a skew matrix. We compare the convergence speeds of extra-gradient algorithms, with or without player sampling. We vary three parameters: the variance $\sigma$ of the noise in the gradient oracle (we add a Gaussian noise on each gradient coordinate), 
the non-smoothness $\lambda$ of the loss, and the skewness of the matrix. We consider small games and large games ($n \in \{5, 50\}$). We use the (simplex-adapted) mirror variant of doubly-stochastic extra-gradient, and a constant stepsize, selected among a grid (see \autoref{app:experiments}). We use variance reduction when $\lambda = 0$ (smooth case). We also consider cyclic sampling in our benchmarks, as described in~\autoref{subsec:algorithms}.

\paragraph{Results.}\autoref{fig:quadratic_convergence} compares the convergence speed of player-sampled extra-gradient for the various settings and sampling schemes. As predicted by \autoref{thm:non_smooth} and \ref{thm:VR}, the regime of convergence in $1 / \sqrt{k}$ in the presence of noise is unchanged with player sampling. DSEG always brings a benefit in the convergence constants (\autoref{fig:quadratic_convergence_5}-\subref{fig:quadratic_convergence_50}), in particular for smooth noisy problems (\autoref{fig:quadratic_convergence_5} center, \autoref{fig:quadratic_convergence_50} left).
Most interestingly, cyclic player selection improves upon random sampling for small number of players (\autoref{fig:quadratic_convergence_5}). 

\autoref{fig:quadratic_convergence_noise} highlights the trade-offs in \autoref{thm:VR}: as the noise increase, the size of player batches should be reduced. Not that for  skew-games with many players (\autoref{fig:quadratic_convergence_50} col. 3), our approach only becomes beneficial in the high-noise regime. As predicted in \autoref{sec:theory}, full EG should be favored with noiseless oracles (see \autoref{app:experiments}).

\begin{figure*}[t]
    \includegraphics[width=\textwidth]{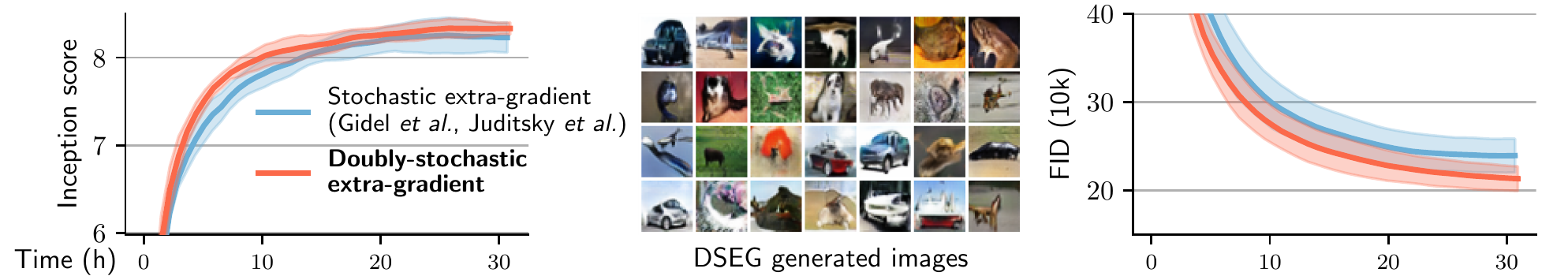}

    \vspace{-.8em}

    \caption{Training curves and samples using doubly-stochastic extragradient on CIFAR10 with WGAN-GP losses, for the best learning rates. Doubly-stochastic extrapolation allows faster and better training, most notably in term of Fréchet Inception Distance (10k). Curves averaged over 5 runs.}\label{fig:gan}
\end{figure*}

\paragraph{Spectral study of sampling schemes.}
The benefit of cyclic sampling can be explained for simple quadratic games.
We consider a two-player quadratic game where $\ell_i(\theta) = {\theta^i}^\top A \theta$ for $i =1, 2$, $\theta = (\theta^1, \theta^2)$ is an unconstrained vector of $\RR^{2 \times d}$, and gradients are noiseless. In this setting, full EG and DSEG expected iterates follows a linear recursion $\EE[\theta_{k+4}] = \Aa(\EE[\theta_{k}])$, where $k$ is the number of gradient evaluation and $\Aa$ is a linear ``algorithm operator'', computable in closed form. A lower spectral radius for $\Aa$ yields a better convergence rate for ${(\EE[\theta_k])}_k$, in light of \citet{gelfand1941normierte} formula---we compare spectral radii across methods.

We sample random payoff matrices $A$ of varying skewness and condition number, and compare the spectral radius $\Aa$ associated to full EG, and DSEG with cyclic and random player selection. As summarized in \autoref{fig:mean_radius}, player sampling reduces the spectral radius of $\Aa$ on average; most interestingly, the reduction is more important using cyclic sampling. Spectral radii are not always in the same order across methods, hinting that sampling can be harmful in the worst cases. Yet cyclic sampling will perform best on average in this (simple) setting. We report details and further figures in~\autoref{sec:spectral_conv}.

\subsection{Generative adversarial networks (GANs)}\label{sec:gans}

We evaluate the performance of the player sampling approach to train a
generative model on \mbox{CIFAR10} \citep{krizhevsky2009learning}. We use the
WGAN-GP loss~\citep{gulrajani_improved_2017}, that defines a non-convex
two-player game. Our theoretical analysis indeed shows a $1 / \sqrt{2}$ speed-up for noisy monotonous 2-player games---the following suggests that speed-up also arises in a non-convex setting. We compare the full stochastic
extra-gradient (SEG) approach advocated by~\citet{gidel2018variational} to the cyclic
sampling scheme proposed in \autoref{subsec:algorithms} (i.e. \textit{extra. D,
upd. G, extra. G, upd. D}). We use the ResNet \citep{he2016deep} architecture
from~\citet{gidel2018variational}, and select the best performing stepsizes
among a grid (see \autoref{app:experiments}). We use the
Adam~\citep{kingma_adam_2014} refinement of
extra-gradient~\citep{gidel2018variational} for both the baseline and proposed
methods. The notion of functional Nash error does not exist in the non-convex
setting. We estimate the convergence speed toward an equilibrium by measuring a
quality criterion for the generator. We therefore evaluate the Inception
Score~\citep{salimans_improved_2016} and Fréchet Inception Distance (FID,
\citet{heusel2017gans} along training, and report their final values.

\paragraph{Results.} We report training curves versus wall-clock time in \autoref{fig:gan}. 
Cyclic sampling allows faster and better training, especially with respect to FID, which is more correlated to human appreciation \citep{heusel2017gans}. 
\autoref{fig:multigan} (right) compares our result to full extra-gradient with uniform averaging. It shows substantial improvements in FID, with results less sensitive to randomness. SEG itself slightly outperforms optimistic mirror descent \citep{gidel2018variational,optimistic}.

\paragraph{Interpretation.}Without extrapolation, alternated training is known to perform better than simultaneous updates in WGAN-GP \citep{gulrajani_improved_2017}. Full extrapolation has been shown to perform similarly to alternated updates \citep{gidel2018variational}. Our approach combine extrapolation with an alternated schedule. It thus performs better than extrapolating with simultaneous updates. It remains true across every learning rate we tested. Echoing our findings of \autoref{sec:quadratic}, deterministic sampling is crucial for performance, as random player selection performs poorly (score 6.2 IS).

\subsection{Mixtures of GANs}

Finally, we consider a simple multi-player GAN setting, akin to
\citet{ghosh_multi-agent_2018}, where $n$ different generators
${(g_{\theta_i})}_i$ seeks to fool $m$ different discriminators ${(f_{\phi_j})}_j$. We
minimize $\sum_j \ell(g_{\theta_i},f_{\phi_j})$ for all $i$, and maximize
$\sum_i \ell(g_{\theta_i},f_{\phi_j})$ for all $j$. Fake data is then sampled
from mixture $\sum_{i=1}^n \delta_{i = J} g_{\theta_i}(\varepsilon)$, where $J$ is
sampled uniformly in $[n]$ and $\varepsilon \sim \Nn(0, I)$. We compare two
methods: (i) SEG extrapolates and updates all
$(g_{\theta_i})_i$, $(f_{\phi_j})_j$ at the same time; (ii) DSEG extrapolates and
updates successive pairs $(g_{\theta_j}, f_{\phi_j})$ alternating the 4-step
updates from \autoref{sec:gans}.

\paragraph{Results.} We compare the training curves of both SEG and DSEG in
\autoref{fig:multigan}, for a range of learning rates. DSEG outperform SEG for
all learning rates; more importantly, higher learning rates can be used for
DSEG, allowing for faster training. DSEG is thus appealing for mixtures of GANs,
that are useful to mitigate mode collapse in generative modeling. We report generated images in Appendix~\ref{app:experiments}.
\begin{figure}[t]
    \begin{minipage}{.5\linewidth}
        \includegraphics[width=\linewidth]{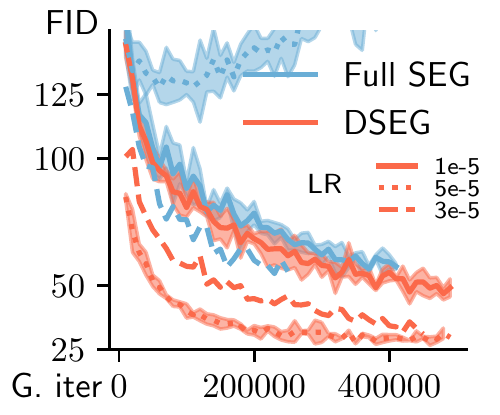}
    \end{minipage}
    \begin{minipage}{.49\linewidth}
        \footnotesize
        \begin{tabular}{lc}
            \toprule
            Method & FID (50k) \\
            \midrule
            SEG & $19.69 \pm 1.53$ \\
            \textbf{DSEG} & $\mathbf{17.10 \pm 1.07}$ \\
            \midrule
             & IS \\
            \midrule
            SEG & $8.26 \pm .16$ \\
            \textbf{DSEG} & $\mathbf{8.38\pm.06}$ \\
            \bottomrule
        \end{tabular}
    \end{minipage}
    \caption{\textit{Left}: Player sampling allows faster training of mixtures of GANs. \textit{Right}: Player sampling trains better ResNet WGAN-GP. FID and IS computed on 50k samples, averaged over 5 runs.}

    \vspace{-1em}

    \label{fig:multigan}
\end{figure}

\section{Conclusion}

%% file: conclusion.tex
We propose and analyse a doubly-stochastic extra-gradient approach for finding Nash equilibria. According to our convergence results, updating and extrapolating random sets of players in extra-gradient brings speed-up in noisy and non-smooth convex problems. Numerically, doubly-stochastic extra-gradient indeed brings speed-ups in convex settings, especially with noisy gradients. It brings speed-ups and \textit{improve solutions} when training non-convex GANs and mixtures of GANs, thus combining the benefits of alternation and extrapolation in adversarial training. Numerical experiments show the importance of \textit{sampling schemes}. We take a first step towards understanding the good behavior of \textit{cyclic} player sampling through spectral analysis.
We foresee interesting developments using player sampling in reinforcement learning: the policy gradients obtained using multi-agent actor critic methods~\citep{lowe2017multi} are noisy estimates, a setting in which it is beneficial.

%% file: appendices.tex
\input{app_existing.tex}
\input{app_lemmas.tex}
\input{app_non_smooth.tex}
\input{app_smooth.tex}

\input{appendix_radius.tex}

\input{appendix_exp.tex}

%% file: app_existing.tex

The appendices are structured as follows: \autoref{sec:existing_res} presents the setting and the existing results. In particular, we start by introducing the setting of the mirror-prox algorithm in \autoref{sec:miror_prox} and detail the relation between solving this problem and finding Nash equilibria in convex $n$-player games \autoref{sec:link_vi_nash}. We then present the proofs of our theorems in \autoref{sec:proofs}. We analyze the DSEG algorithm (\autoref{alg:doubly_stoch}) and study its variance-reduction version. \autoref{app:experiments} presents further experimental results and details.
 
\printcontents[sections]{l}{1}{\setcounter{tocdepth}{2}}

\section{Existing results}\label{sec:existing_res}

\subsection{Mirror-prox} \label{sec:miror_prox}

Mirror-prox and mirror descent are the formulation of the extra-gradient method and gradient descent for non-Euclidean (Banach) spaces. \citet{bubeck_monograph} (which is a good reference for this subsection) and \citet{solving} study extra-gradient/mirror-prox in this setting. We provide an introduction to the topic for completeness.

\paragraph{Setting and notations.}
We consider a Banach space $E$ and a 
compact set $\Theta \subset E.$ We define an open convex set $\mathcal{D}$ such that $\Theta$ is included in its closure, that is $\Theta\subseteq \bar{\mathcal{D}}$ and $\mathcal{D}\cap \Theta\neq \emptyset.$ 
The Banach space $E$ is characterized by a norm $\|\cdot\|$. Its conjugate norm $\|\cdot\|_*$ is defined as $\|\xi \|_* = \max_{z:\|z\|\leq 1}\langle \xi,z\rangle.$
For simplicity, we assume $E=\mathbb{R}^n$.

We assume the existence of a mirror map for $\Theta$, which is defined as a function $\Phi\colon \mathcal{D}\rightarrow\mathbb{R}$ that is differentiable and $\mu$-strongly convex i.e. 
\begin{equation}
    \forall x,y\in \mathcal{D}, \; \langle \nabla \Phi(x)-\nabla \Phi(y),x-y\rangle \geq \mu \|x-y\|^2.
\end{equation}
We can define the \textit{Bregman divergence} in terms of the mirror map.

\begin{definition}
Given a mirror map $\Phi\colon\mathcal{D}\rightarrow \mathbb{R}$, the Bregman divergence $D: 
\mathcal{D} \times \mathcal{D}
\rightarrow \mathbb{R}$ is defined as
\begin{equation}
    D
    (x,y) \eqdef \Phi(x) - \Phi(y) - \langle \nabla \Phi(y), x-y \rangle.
\end{equation}
\end{definition}

Note that $D
(\cdot,\cdot)$ is always non-negative. For more properties, see e.g. \citet{nemirovsky1983problem} and references therein. Given that $ \Theta$ is compact convex space, we define $\Omega = \max_{x\in\mathcal{D}\cap \Theta}\Phi(x)-\Phi(x_1).$ Lastly, for $z\in \mathcal{D}$ and $\xi \in E^*,$ we define the prox-mapping as
\begin{equation}\label{eq:prox_mapping}
     P_z(\xi) \eqdef \argmin_{u \in \mathcal{D} \cap \Theta} \{\Phi(u) + \langle \xi - \nabla \Phi(z), u \rangle\} = \argmin_{u \in \mathcal{D} \cap \Theta} \{D
     (z,u) + \langle \xi, u \rangle\}.
\end{equation}

The mirror-prox algorithm is the most well-known algorithm to solve convex $n$-player games in the mirror setting (and variational inequalities, see~\autoref{sec:link_vi_nash}). 
An iteration of mirror-prox consists of:
\begin{equation}\label{eq:mirror_prox}
    \begin{split}
     \text{Compute the extrapolated point: }&\begin{cases}
     \nabla \Phi (y_{\tau +1/2}) = \nabla \Phi (\theta_\tau) - \gamma F(\theta_\tau),\\
      \theta_{\tau + 1/2} = \argmin_{x \in \mathcal{D} \cap \Theta} \ D(x, y_{\tau+1/2}),
    \end{cases}\\
    \text{Compute a gradient step: }&\begin{cases}
    \nabla \Phi (y_{\tau +1}) = \nabla \Phi (\theta_\tau) - \gamma F(\theta_{\tau+1/2}),\\
    \theta_{\tau + 1} = \argmin_{x \in \mathcal{D} \cap \Theta} \ D(x, y_{\tau+1}).
    \end{cases}.
    \end{split}
\end{equation}
Remark that the extra-gradient algorithm defined in equation~\eqref{eq:extra_grad} corresponds to the mirror-prox~\eqref{eq:mirror_prox} when choosing $\Phi(x) = \frac{1}{2}\|x\|_2^2.$ 

\begin{lemma}
By using the proximal mapping notation~\eqref{eq:prox_mapping}, the mirror-prox updates are equivalent to: 
\begin{equation}\label{eq:mirror_prox2}
    \begin{split}
     \text{Compute the extrapolated point: }&  \theta_{\tau + 1/2} = P_{\theta_\tau}(\gamma F(\theta_\tau)),\\
    \text{Compute a gradient step: }& \theta_{\tau + 1} = P_{\theta_\tau}(\gamma F(\theta_{\tau+1/2})).
    \end{split}
\end{equation}
\end{lemma}

\begin{proof}
We just show that $\theta_{\tau + 1/2} = P_{\theta_\tau}(\gamma F(\theta_\tau))$, as the second part is analogous.
\begin{equation*}\label{eq:proof_mp}
    \begin{split}
    \theta_{\tau+1/2} &= \argmin_{x \in \mathcal{D} \cap \Theta} \ D(x, y_{\tau+1/2})\\
    &= \argmin_{x \in \mathcal{D} \cap \Theta} \  \Phi(x) - \langle \nabla \Phi(y_{\tau+1/2}), x \rangle \\ 
    &= \argmin_{x \in \mathcal{D} \cap \Theta} \  \Phi(x) - \langle \nabla \Phi(\theta_{\tau}) - \alpha F(\theta_{\tau}), x \rangle\\
    &= \argmin_{x \in \mathcal{D} \cap \Theta} \ \langle \alpha F(\theta_\tau), x \rangle + D(x,\theta_\tau).\qedhere
    \end{split}
\end{equation*}
\end{proof}

The mirror framework is particularly well-suited for simplex constraints i.e. when the parameter of each player is a probability vector. Such constraints usually arise in matrix games. If $\Theta_i$ is the $d_i$-simplex, we express the negative entropy for player $i$ as
\begin{equation}
    \Phi_i(\theta^{i}) = \sum_{j=1}^{d_i} \theta^{i}(j) \log \theta^{i}(j). 
\end{equation}
We can then define $\mathcal{D}\eqdef \interi \Theta = \interi \Theta_1 \times \cdots \times \interi \Theta_n$ and the mirror map as 
\begin{equation}
    \Phi(\theta) = \sum_{i=1}^n \Phi_i(\theta^{i}).
\end{equation}
We use this mirror map in the experiments for random monotone quadratic games (\autoref{sec:quadratic}).

\subsection{Link between convex games and variational inequalities}\label{sec:link_vi_nash}

As first noted by \citet{rosen_existence_1965}, finding a Nash equilibrium in a convex $n$-player game is related to solving a variational inequality (VI) problem. We consider a space of parameters $\Theta \subseteq \mathbb{R}^d$ that is compact and convex, equipped with the standard scalar product $\langle \cdot, \cdot \rangle$ in $\mathbb{R}^d$.

For convex $n$-player games (\autoref{ass:cvx_game}), the simultaneous (sub)gradient $F$ (Eq.~\ref{eq:sim_grad}) is a monotone operator.

\begin{definition}
    An operator $F\colon \Theta \rightarrow \mathbb{R}^d$ is monotone if $\forall \theta,\theta'\in \Theta,\; \langle F(\theta)-F(\theta'),\theta-\theta'\rangle \geq 0.$
\end{definition}

Assuming continuity of the losses $\ell_i$, we then consider the set of solutions to the following vairational inequality problem:
\begin{equation}\label{eq:weakvip}
    \text{Find }\; \theta_{*} \in \Theta \;\text{ such that }\; \langle F(\theta),\theta-\theta_{*} \rangle \geq 0\quad  \forall \theta \in \Theta.
\end{equation}
Under \autoref{ass:cvx_game}, this set coincides with the set of Nash equilibria, and we may solve \eqref{eq:weakvip} instead of \eqref{def:nash}~\citep{rosen_existence_1965,harker1990finite,nemirovski2010accuracy}. \eqref{eq:weakvip} indeed corresponds to the first-order necessary optimality condition applied to the loss of each player.  

The quantity used to quantify the inaccuracy of a solution $\theta$ to \eqref{eq:weakvip}
is the dual VI gap defined as $\mathrm{Err}_{\mathrm{VI}}(\theta) = \max_{u\in \Theta} \langle F(u),\theta-u \rangle.$ However, the \textit{functional Nash error} \eqref{eq:nash_err}, also known as the \citep{nikaido_note_1955} function, is the usual performance measure for convex games. We provide the convergence rates in term of functional Nash error but they also apply to the dual VI gap.

%% file: app_lemmas.tex
 \clearpage
 \section{Proofs and mirror-setting algorithms} \label{sec:proofs}
 
 We start by proving \autoref{cor:juditsky_play_sampl}, that derives from \citet{solving} (\autoref{sec:rand_noise}). As this result is not instructive, we use the structure of the player sampling noise in~\eqref{eq:noisy_sim_grad}
  to obtain a stronger result in the non-smooth case (\autoref{sec:doub_sto_mirr_prox}). For this, we directly modify the proof of \autoref{juditsky_thm_grad_comp} from \citet{solving}, using a few useful lemmas 
  (\autoref{sec:useful_lem}). We then turn to the smooth case, for which a variance reduction mechanism proves necessary (\autoref{sec:doub_sto_vr}). The proof is original, and builds upon techniques from the variance reduction literature \citep{defazio2014saga}.

 \subsection{Proof of \autoref{cor:juditsky_play_sampl}} \label{sec:rand_noise}
Player sampling noise modifies the variance of the unbiased gradient estimate. Indeed, in equation~\eqref{eq:noisy_sim_grad} $\tilde{F}_i(\theta,\mathcal{P})$ is an unbiased estimate of $\nabla_i \ell_i(\theta)$, and for all $i \in [n]$
 \begin{equation}
    \E{\tilde{F}_i(\theta,\mathcal{P})} = \text{Prob}(i \in \mathcal{P}) \frac{n}{b} \E{g_i(\theta)} = \E{g_i(\theta)} = \nabla_i \ell_i(\theta).
 \end{equation}
 If $g_i$ has variance bounded by $\sigma^2$, we can bound the variance of $\tilde{F}_i(\theta,\mathcal{P})$:
 \begin{align} \label{eq:var_tilde}
    \E{\|\tilde{F}_i(\theta,\mathcal{P})-\nabla_i \ell_i(\theta)\|^2} & = \E{\|\tilde{F}_i(\theta,\mathcal{P})-g_i(\theta) + g_i(\theta) - \nabla_i \ell_i(\theta)\|^2} \\ &\leq 2\E{\|\tilde{F}_i(\theta,\mathcal{P})-g_i(\theta)\|^2} + 2\E{\|g_i(\theta) - \nabla_i \ell_i(\theta)\|^2}\\ &\leq 2\E{\|\tilde{F}_i(\theta,\mathcal{P})-g_i(\theta)\|^2} + 2\sigma^2 \\ &= 2\E{\frac{b}{n}\left\|\left(\frac{n}{b}-1 \right)g_i(\theta)\right\|^2 + \left(1-\frac{b}{n}\right)\left\|g_i(\theta)\right\|^2} + 2\sigma^2 \\ &\leq 2\frac{n-b}{b} \E{\|g_i(\theta)\|^2} + 2\sigma^2\\
    &\leq 2\frac{n-b}{b} G^2 +2 \sigma^2.
 \end{align}
 Substituting $\sigma^2$ by $2\frac{n-b}{b} G^2 +2 \sigma^2$ in equations~\eqref{a1juditsky} and~\eqref{a2juditsky} yields:
 \begin{equation} \label{a1juditskypl}
     \E{ \mathrm{Err}_N(\hat{\theta}_{t(k)}) } \leq 14 n \sqrt{\frac{\Omega}{3k}\left(\frac{4n-3b}{b}G^2 + 2\sigma^2 \right)} = \Oo\left(n\sqrt{\frac{\Omega }{k}\left(\frac{n}{b}G^2 + \sigma^2\right)}\right). 
 \end{equation}
 \begin{align} \label{a2juditskypl}
     \E{ \mathrm{Err}_N(\hat{\theta}_{t(k)}) } &\leq \max \left\{ \frac{7 \Omega L n^{3/2}}{k}, 28 n \sqrt{\frac{\Omega ((\frac{n}{b}-1) G^2 + \sigma^2)}{3k}}\right\}
 \end{align}
 These bounds are worse than the ones in \autoref{juditsky_thm_grad_comp} when $b \ll n$. This motivates the following derivations, that yields~\autoref{thm:non_smooth} and \ref{thm:VR}.

 \subsection{Useful lemmas}\label{sec:useful_lem}
The following two technical lemmas are proven and used in the proof of Theorem 2 of \citet{solving}. 
 \begin{lemma}\label{lem:lemma4}
  Let $z$ be a point in $\mathcal{X}$, let $\chi,\eta$ be two points in the dual $E^{*}$, let $w = P_z(\chi)$ and $r_+=P_z(\eta).$ Then,
   \begin{equation}
      \|w-r_+\|\leq \|\chi-\eta\|_*~.
  \end{equation}
  Moreover, for all $u\in E$, one has
  \begin{equation}
      D(u,r_+) - D(u,z) \leq \langle \eta,u-w\rangle + \frac{1}{2}\|\chi-\eta\|_*^2-\frac{1}{2}\|w-z\|^2~.
  \end{equation}
 \end{lemma}
 
 \begin{lemma}\label{lem:corollary2}
  Let $\xi_1,\xi_2,\dots$ be a sequence of elements of $E^{*}$. Define the sequence $\{ y_{\tau}\}_{\tau=0}^{\infty}$ in $\mathcal{X}$ as follows: 
  \[ y_{\tau}=P_{y_{\tau-1}}(\xi_{\tau}).\]
  
  Then $y_{\tau}$ is a measurable function of $y_0$ and $\xi_1,\dots,\xi_{\tau}$ such that: 
  \begin{equation}
      \forall u \in Z, \qquad \bigg\langle \sum_{\tau=1}^t \xi_t, y_{\tau-1}-u \bigg\rangle \leq D(u,y_0) + \frac{1}{2}\sum_{\tau=1}^t \|\xi_{\tau}\|_*^2.
  \end{equation}
 \end{lemma}
 
 The following lemma stems from convexity assumptions on the losses (\autoref{ass:cvx_game}) and is proven as an intermediate development of the proof of Theorem 2 of \citet{solving}.
 
 \begin{lemma} \label{cneplemma}
 We consider a convex $n$-player game with players losses $\ell_i$ where $i\in [n]$. Let a sequence of points $(z_\tau)_{\tau \in [t]} \in \Theta$, the stepsizes $(\gamma_\tau)_{\tau \in [t]} \in (0,\infty)$. We define the average iterate $\hat{z}_\tau = \left[ \sum_{\tau=0}^t \gamma_{\tau}\right]^{-1}\sum_{\tau=0}^t \gamma_{\tau}z_{\tau}$. The functional Nash error evaluated in $\hat{z}_t$ is upper bounded by
 \begin{equation}
     \mathrm{Err}_N(\hat{z}_t) \eqdef \sup_{u \in Z} \sum_{i=1}^n \ell_i(\hat{z}_t)-\ell_i(u^{i},\hat{z}_t^{-i}) \leq \sup_{u \in Z} \left(\sum_{\tau = 0}^t \gamma_\tau\right)^{-1} \sum_{\tau=0}^t \langle \gamma_\tau F(z_\tau), z_\tau - u \rangle.
 \end{equation}
 \end{lemma}
 
The following lemma is a consequence of first-order optimality conditions.
 
 \begin{lemma} \label{equal_terms}
 Let $(\gamma_t)_{t\in \mathbb{N}}$ be a sequence in $(0,\infty)$ and $A, B > 0$. For any $t\in \mathbb{N}$, we define the function~$f_t$ to be
 \begin{equation}
     f_t(\alpha) \eqdef \frac{A}{\sum_{\tau = 0}^t \alpha \gamma_\tau} + \frac{B \sum_{\tau = 0}^t (\alpha \gamma_\tau)^2}{\sum_{\tau = 0}^t \alpha \gamma_\tau}.
 \end{equation}
 Then, it attains its minimum for $\alpha > 0$ when both terms are equal. Let us call $\alpha_{*}$ the point at which the minimum is reached. The value of $f_t$ evaluated at $\alpha_{*}$ is
 \begin{equation}
     f_t(\alpha_{*})=f\left(\sqrt{\frac{A}{B\sum_{\tau=0}^t\gamma_\tau^2}}\right)=\frac{2\sqrt{AB\sum_{\tau = 0}^t \gamma_\tau^2}}{\sum_{\tau = 0}^t \gamma_\tau}.
 \end{equation}
 \end{lemma}

 The next lemma describes the dual norm of the natural Pythagorean norm on a Cartesian product of Banach spaces.
 
 \begin{lemma} \label{dual}
 Let $(X_1,\|\cdot\|_{X_1}),\dots,(X_n,\|\cdot\|_{X_n})$ be Banach spaces where for each $i$, $\|\cdot\|_{X_i}$ is the norm associated to $X_i$. The Cartesian product is $X = X_1 \times X_2 \times \dots \times X_n$ and has a norm $\|\cdot\|_X$ defined for $y=(y_1,\dots,y_n)\in X$ as
 \begin{equation}
    \|y\|_X \eqdef \sqrt{\sum_{i=1}^n \|y_i\|_{X_i}^2}.
 \end{equation}
 It is known that $(X, \|\cdot\|_X)$ is a Banach space. Moreover, we define the dual spaces $(X_1^*,\|\cdot\|_{X_1^{*}},\dots,(X_n^*,\|\cdot\|_{X_n^{*}}).$ The dual space of $X$ is  $X^{*} = X_1^{*} \times X_2^{*} \times ... \times X_n^{*}$ and has a norm $\|\cdot\|_{X^{*}}$. Then, for any $a=(a_1,...,a_n) \in X^{*}$, the following inequality holds
 \begin{equation}
     \|a\|_{X^{*}}^2 = \sum_{i=1}^n \|a_i\|_{X_i^{*}}^2.
 \end{equation}
 \end{lemma}
 
 \begin{proof}
 
 On the one hand,
 \begin{equation}
     \|a\|_{X^{*}}^2 = \sup_{y \in X} \frac{|ay|^2}{\|y\|_{X}^2} =  \sup_{y \in X} \frac{\left(\sum_{i=1}^n a_i y_i \right)^2}{\|y\|_{X}^2} \leq \sup_{y \in X} \frac{\bigg(\sum_{i=1}^n \|a_i\|_{X_i^{*}} \|y_i\|_{X_i} \bigg)^2}{\|y\|_{X}^2},\label{eq:1st_ineq_lem6}
 \end{equation}
 and by Cauchy-Schwarz inequality
 \begin{equation}
     \|a\|_{X^{*}}^2 \leq \sup_{y \in X} \frac{\left(\sum_{i=1}^n \|a_i\|_{X_i^{*}}^2\right) \left(\sum_{i=1}^n \|y_i\|_{X_i}^2\right)}{\|y\|_{X}^2}= \sum_{i=1}^n \|a_i\|_{X_i^{*}}^2.
 \end{equation}
 To prove the other inequality we define $Z_i = \left\{y_i \in X_i | \|y_i\|_{X} = \|a_i\|_{X_i^{*}}\right\}$.
 \begin{equation}
     \|a\|_{X^{*}}^2 \geq \sup_{y \in Z_1\times \cdots \times Z_n} \frac{|ay|^2}{\|y\|_{X}^2}
     = \frac{\left(\sum_{i=1}^n \sup_{y_i \in Z_i} a_i y_i \right)^2}{\sum_{i=1}^n \|a_i\|_{X_i^{*}}^2}
     = \frac{\left(\sum_{i=1}^n \|a_i\|_{X_i^{*}}^2 \right)^2}{\sum_{i=1}^n \|a_i\|_{X_i^{*}}^2} = \sum_{i=1}^n \|a_i\|_{X_i^{*}}^2 .
 \end{equation}%
 \end{proof}
 
The following two numerical lemmas will be used in \autoref{lemm:equalities2}.
\begin{lemma} \label{longineq}
The following inequality holds for any $j \in \mathbb{N}, p \in \mathbb{R}$ such that $p>0$: 
\begin{equation} 
    \frac{(2\ceil{(j+1)/2}-j) (1-p)^{2\ceil{(j+1)/2}-j-1} p+2(1-p)^{2\ceil{(j+1)/2}-j}}{p^2} \leq \frac{2-p}{p^2}.
\end{equation}
\end{lemma}

\begin{proof}
For $j$ even, we can write
\begin{equation} 
    (2\ceil{(j+1)/2}-j) (1-p)^{2\ceil{(j+1)/2}-j-1} p+2(1-p)^{2\ceil{(j+1)/2}-j} = 2 (1-p) p+2(1-p)^{2} = 2(1-p).
\end{equation}
For $j$ odd,
\begin{equation} 
    (2\ceil{(j+1)/2}-j) (1-p)^{2\ceil{(j+1)/2}-j-1} p+2(1-p)^{2\ceil{(j+1)/2}-j} = p+1-p+1-p = 2-p.
\end{equation}
Since $p > 0$, $2-p \geq 2(1-p)$.
\end{proof}

\begin{lemma} \label{lemm:sum_deriv}
For all $|\alpha|<1$,
\begin{equation}
    \sum_{s = q}^{\infty} \alpha^{s-1} s = \frac{q \alpha^{q-1} (1-\alpha)+\alpha^q}{(1-\alpha)^2}.
\end{equation}
\end{lemma}

\begin{proof}
\begin{equation*}
\sum_{s = q}^{\infty} \alpha^{s-1} s = \left(\sum_{s = q}^{\infty} \alpha^s\right)' = \left(\frac{\alpha^q}{1-\alpha}\right)' = \frac{q \alpha^{q-1} (1-\alpha)+\alpha^q}{(1-\alpha)^2}.\qedhere
\end{equation*}
\end{proof}

%% file: app_non_smooth.tex

\subsection{Doubly-stochastic mirror-prox---Proof of \autoref{thm:non_smooth}}
\label{sec:doub_sto_mirr_prox}



\subsubsection{Algorithm}

While~\autoref{alg:doubly_stoch} presents the doubly-stochastic algorithm in the Euclidean setting, we consider here its mirror version.

\begin{algorithm}
\caption{Doubly-stochastic mirror-prox} 
\label{alg:doubly_stoch_mirror}
\begin{algorithmic}[1]
\State \textbf{Input}: initial point $\theta_{0}\in \mathbb{R}^{d},$  stepsizes $(\gamma_{\tau})_{\tau \in [t]}$, mini-batch size over the players $b\in [n]$.

\For {$\tau=0,\dots,t$}
        \State Sample the random matrices $M_{\tau},M_{\tau+1/2}\in \RR^{d\times d}$.
        \State Compute $\tilde{F}_{\tau+1/2}= \frac{n}{b}\cdot M_{\tau}\hat{F}(\theta_{\tau}).$
        \State Extrapolation step:  $\theta_{\tau+1/2} = 
         P_{\theta_{\tau}}(\gamma_{\tau}\tilde{F}_{\tau+1/2})$.
         \State Compute $\tilde{F}_{\tau+1}= \frac{n}{b}\cdot M_{\tau+1/2}\hat{F}(\theta_{\tau+1/2}).$
        \State Gradient step: $\theta_{\tau+1} = 
         P_{\theta_{\tau}}(\gamma_{\tau}\tilde{F}_{\tau+1}).$
\EndFor
\State \textbf{Return} $\hat{\theta}_t = \left[ \sum_{\tau=0}^t \gamma_{\tau}\right]^{-1}\sum_{\tau=0}^t \gamma_{\tau}\theta_{\tau}.$
\end{algorithmic}
\end{algorithm}

\paragraph{Notation.}

We introduce the noisy simultaneous gradient $\hat{F}(\theta)$ defined as
\begin{equation}\label{eq:noisy_sim_grad_mirror}
    \hat{F}(\theta) = (\hat{F}^{(1)}(\theta),\dots,\hat{F}^{(n)}(\theta))^{\top} \eqdef (g_1,\dots,g_n)^{\top} \in \mathbb{R}^d, 
\end{equation}
where $g_i$ is a noisy unbiased estimate of $\nabla_i l_i(\theta)$ with variance bounded by $\sigma^2$. We are abusing the notation because $\hat{F}(\theta)$ is a random variable indexed by $\Theta$ and not a function, but we do so for the sake of clarity.

For our convenience, we also define the ratio $p=b/n.$

\paragraph{Differences with~\autoref{alg:doubly_stoch}}

The notation in \autoref{alg:doubly_stoch_mirror} differs in a few aspects. First, we model the sampling over the players by using the random block-diagonal matrices $M_{\tau}$ and $M_{\tau+1/2}$ in $\mathbb{R}^{d\times d}.$ More precisely, at each iteration, we select according to a uniform distribution $b$ diagonal blocks and assign them to the identity matrix. Remark that we add a factor $n/b$ in front of the random matrices to ensure the unbiasedness of the gradient estimates $ \tilde{F}_{\tau}$ and $ \tilde{F}_{\tau+1/2}.$ Note that the matrices $M_{\tau}$ and $M_{\tau+1/2}$ are just used for the convenience of the analysis. In practice, sampling over players is not performed in this way. 

Moreover, while the update in \autoref{alg:doubly_stoch} involve Euclidean projections, we use the proximal mapping~\eqref{eq:prox_mapping} in~\autoref{alg:doubly_stoch_mirror}. The new notation will be used throughout the appendix.

We first proceed to the analysis of~\autoref{alg:doubly_stoch_mirror} in the case of non-smooth losses. 

\subsubsection{Convergence rate under Assumption~\ref{ass:non-smooth} (non-smoothness)---proof of \autoref{thm:non_smooth}} \label{non_smooth_proof}
The following \autoref{thmstoch} generalizes \autoref{thm:non_smooth} to the mirror setting.
\begin{theorem} \label{thmstoch}
We consider a convex $n$-player game where~\autoref{ass:non-smooth} holds. 
Assume that~\autoref{alg:doubly_stoch_mirror} is run with  constant stepsizes $\gamma_{\tau} = \gamma$. Let $t(k) = k/(2b)$ be the number of iterations corresponding to $k$ gradient computations. Setting 
\begin{equation}
 \gamma = \sqrt{\frac{2\Omega}{n\left(\frac{(3n-b)G^2}{b}+\sigma^2 \right)t(k)}},
 \end{equation}
 the rate of convergence in expectation at iteration $t(k)$ is 
\begin{equation} \label{minconstc}
     \E{ \mathrm{Err}_N(\hat{\theta}_{t(k)}) }
      = 4\sqrt{\frac{\Omega n \left(3G^2n + b(\sigma^2-G^2)\right)}{k}}.
 \end{equation}
\end{theorem}

\begin{proof}
The strategy of the proof is similar to the proof of Theorem 2 and part of Theorem 1 from \citet{solving}. It consists in bounding $\sum_{\tau=0}^t  \langle \gamma_{\tau}F(\theta_{\tau+1/2}),\theta_{\tau+1/2}-u\rangle$, which by \autoref{cneplemma} is itself a bound of the functional Nash error.

By using \autoref{lem:lemma4} with $z=\theta_{\tau}$, $\chi = \gamma_{\tau}\tilde{F}_{\tau+1/2}$, $\eta = \gamma_{\tau}\tilde{F}_{\tau+1}$ (so that $w=\theta_{\tau+1/2}$ and $r_+=\theta_{\tau+1}$), we have for any $u\in \Theta$ 
 \begin{align} \label{eq:lemma1_application_1stproof}
     \langle \gamma_{\tau}\tilde{F}_{\tau+1},\theta_{\tau+1/2}-u\rangle + D(u,\theta_{\tau+1})-D(u,\theta_{\tau})&\leq \frac{\gamma_{\tau}^2}{2}\|\tilde{F}_{\tau+1}-\tilde{F}_{\tau+1/2}\|_{*}^2-\frac{1}{2}\|\theta_{\tau+1/2}-\theta_{\tau}\|_*^2\nonumber\\
    &\leq \frac{\gamma_{\tau}^2}{2}\|\tilde{F}_{\tau+1}-\tilde{F}_{\tau+1/2}\|_{*}^2.
 \end{align}
 When summing up from $\tau=0$ to $\tau=t$ in equation~\eqref{eq:lemma1_application_1stproof}, we get \begin{equation} \label{2ndeq}
     \sum_{\tau=0}^t  \langle \gamma_{\tau}\tilde{F}_{\tau+1},\theta_{\tau+1/2}-u\rangle \leq D(u,\theta_{0})-D(u,\theta_{t+1}) + \sum_{\tau=0}^t \frac{\gamma_{\tau}^2}{2}\|\tilde{F}_{\tau+1}-\tilde{F}_{\tau+1/2}\|_{*}^2.
 \end{equation}
By decomposing the right-hand side \eqref{2ndeq}, we obtain
 \begin{equation} \label{basiceq}
 \begin{split}
     \sum_{\tau=0}^t  \langle \gamma_{\tau}F(\theta_{\tau+1/2}),\theta_{\tau+1/2}-u\rangle&\leq  D(u,\theta_{0})-D(u,\theta_{t+1}) + \sum_{\tau=0}^t\frac{\gamma_{\tau}^2}{2}\|\tilde{F}_{\tau+1}-\tilde{F}_{\tau+1/2}\|_{*}^2\\
          &+\sum_{\tau=0}^t  \bigg\langle\gamma_{\tau}(F(\theta_{\tau+1/2})-\tilde{F}_{\tau+1}),\theta_{\tau+1/2}-u\bigg\rangle\\
          &\leq \Omega +\sum_{\tau=0}^t\frac{\gamma_{\tau}^2}{2}\|\tilde{F}_{\tau+1}-\tilde{F}_{\tau+1/2}\|_{*}^2\\
          &+\sum_{\tau=0}^t \gamma_{\tau} \bigg\langle F(\theta_{\tau+1/2})-\tilde{F}_{\tau+1},\theta_{\tau+1/2}-y_{\tau}\bigg\rangle\\
          &+  \sum_{\tau=0}^t \gamma_{\tau} \bigg\langle F(\theta_{\tau+1/2})-\tilde{F}_{\tau+1},y_{\tau}-u\bigg\rangle,
 \end{split}
 \end{equation}
 where we used $D(u,\theta_0)\leq \Omega$ and defined  $y_{\tau+1}=P_{y_{\tau}}(\gamma_{\tau}\Delta_{\tau})$ with $y_{0} = \theta_0$ and  $\Delta_{\tau}=F(\theta_{\tau+1/2})-\tilde{F}_{\tau+1}.$ So far, we followed the same steps as \citet{solving}. We aim at bounding the left-hand side of equation~\eqref{basiceq} in expectation. To this end, we will now bound the expectation of each of the right-hand side terms. These steps represent the main difference with the analysis by \citet{solving}.

We first define the filtrations 
$\mathcal{F}_{\tau} = \sigma(\theta_{\tau'} : \tau' \leq \tau + 1/2)$
and 
$\mathcal{F}_{\tau} = \sigma(\theta_{\tau'} : \tau' \leq \tau)$.
We now bound the third term on the right-hand side of \eqref{basiceq} in expectation. 
\begin{align} \label{tildeexp}
    \E{ \|\tilde{F}_{\tau+1}-\tilde{F}_{\tau+1/2}\|_{*}^2} &\leq 
    2\left(\E{\|\tilde{F}_{\tau+1}\|_{*}^2}+\E{\|\tilde{F}_{\tau+1/2}\|_{*}^2} \right)\nonumber\\
    &= \frac{2}{p^2}\left(\E{\E{\|M_{\tau+1/2}\hat{F}(\theta_{\tau+1/2})\|_{*}^2|\mathcal{F}_{\tau}}}+\E{\E{\|M_{\tau}\hat{F}(\theta_{\tau})\|_{*}^2|\mathcal{F}_{\tau}'}}\right) \nonumber\\
    &= \frac{2}{p^2}\sum_{i=1}^n\left(\E{\E{\|M_{\tau+1/2}^{(i)}\hat{F}^{(i)}(\theta_{\tau+1/2})\|_{*}^2|\mathcal{F}_{\tau}}}\right.\\
    &\left.\qquad +\E{\E{\|M_{\tau}^{(i)}\hat{F}^{(i)}(\theta_{\tau})\|_{*}^2|\mathcal{F}_{\tau}'}}\right)\nonumber\\
    &\leq \frac{2}{p}\sum_{i=1}^n\E{\|\hat{F}^{(i)}(\theta_{\tau+1/2})\|_{*}^2}+\E{\|\hat{F}^{(i)}(\theta_{\tau})\|_{*}^2}\\\nonumber 
    &\leq \frac{4 n G^2}{p}, 
\end{align}

where we used $\|a+b\|_*^2\leq 2\|a\|_*^2+2\|b\|_*^2$ in the first inequality and applied \autoref{dual} in the second equality. 
Now, we compute the expectation of the fourth term of equation~\eqref{basiceq}.
 \begin{align} \label{eq:unbiased_dotprod_1stproof}
        &\E{\gamma_{\tau} \sum_{\tau=0}^t  \bigg\langle F(\theta_{\tau+1/2})-\tilde{F}_{\tau+1},y_{\tau}-u\bigg\rangle}\\
        &= \E{\sum_{\tau=0}^t  \E{\bigg\langle\gamma_{\tau}\left(I-\frac{M_{\tau+1/2}}{p}\right)\hat{F}(\theta_{\tau+1/2}),\theta_{\tau+1/2}-y_{\tau}\bigg\rangle\bigg|\mathcal{F}_{\tau}}}\nonumber \\
         &= \E{\sum_{\tau=0}^t\bigg\langle \gamma_{\tau} \E{\left(I-\frac{M_{\tau+1/2}}{p}\right)\bigg|\mathcal{F}_{\tau}}\E{\hat{F}(\theta_{\tau+1/2})\bigg|\mathcal{F}_{\tau}},\theta_{\tau+1/2}-y_{\tau}\bigg\rangle}\nonumber\\
         &= 0,
 \end{align}
where we used the independence property of the random variables in the second equality and $\mathbb{E}[\frac{k}{n}\cdot M_{\tau+1/2}]=I_d$ in the third equality. Regarding the fifth term of~\eqref{basiceq}, by using the sequences $\{y_{\tau}\}$ and $\{\xi_{\tau} = \gamma_{\tau}\Delta_{\tau}\}$ in \autoref{lem:corollary2} (as done in \citet{solving}), we obtain: 
\begin{equation} \label{eq:result_lemma23}
    \sum_{\tau=0}^t \langle \gamma_{\tau}\Delta_{\tau},y_{\tau}-u\rangle \leq D(u,\theta_{0}) + \sum_{\tau=0}^t \frac{\gamma_{\tau}^2}{2}\| \Delta_{\tau}\|_{*}^2
    \leq 
    \Omega
    +  \sum_{\tau=0}^t \frac{\gamma_{\tau}^2}{2}\| F(\theta_{\tau+1/2})-\tilde{F}_{\tau+1}\|_{*}^2. 
\end{equation}
We now bound the expectation of $\|F(\theta_{\tau+1/2})-\tilde{F}_{\tau+1}\|_{*}^2$ using the filtration $\mathcal{F}_{\tau}$. By using \autoref{dual} in the first equality, $\|a+b\|_*^2 \leq 2\|a\|_*^2+2\|b\|_*^2$ in the second inequality and the bound on the variance (\autoref{ass:unbiased_var_bnd}) in the third inequality, we obtain
\begin{align} 
&\E{\|F(\theta_{\tau+1/2})-\tilde{F}_{\tau+1}\|_{*}^2} \\ &=  \sum_{i=1}^n \E{\|F^{(i)}(\theta_{\tau+1/2})-\tilde{F}^{(i)}_{\tau+1}\|_{*}^2}\\
&= \sum_{i=1}^n \E{\bigg\|F^{(i)}(\theta_{\tau+1/2}) - \frac{M_{\tau+1}^{(i)}}{p}\hat{F}^{(i)}(\theta_{\tau+1/2})\bigg\|_{*}^2}\\
&\leq \sum_{i=1}^n 2 \E{\bigg\|\left(I-\frac{M^{(i)}_{\tau+1}}{p}\right)\hat{F}^{(i)}(\theta_{\tau+1/2})\bigg\|_{*}^2} +\sum_{i=1}^n 2 \E{
\bigg\|F^{(i)}(\theta_{\tau+1/2})-\hat{F}^{(i)}(\theta_{\tau+1/2})\bigg\|_{*}^2}
\\ &\leq \sum_{i=1}^n 2\E{p \bigg\|\frac{p-1}{p}\hat{F}^{(i)}(\theta_{\tau+1/2}) \bigg\|_{*}^2
+ (1-p)\|\hat{F}^{(i)}(\theta_{\tau+1/2}) \|_{*}^2}
+ 2n\sigma^2\\
&= \sum_{i=1}^n 2\left(1-p+\frac{(1-p)^2}{p}\right)\E{\|\hat{F}^{(i)}(\theta_{\tau+1/2}) \|_{*}^2} + 2n\sigma^2 \\
&= \sum_{i=1}^n 2\left(\frac{1}{p}-1\right)\E{\|\hat{F}^{(i)}(\theta_{\tau+1/2}) \|_{*}^2} + 2n\sigma^2\\
&\leq \frac{2nG^2(1-p)}{p} + 2n\sigma^2. 
\label{initial24}
\end{align}
Therefore, by taking the expectation in equation~\eqref{basiceq} and plugging~\eqref{tildeexp},~\eqref{eq:unbiased_dotprod_1stproof},~\eqref{eq:result_lemma23} and~\eqref{initial24}, we finally get: 
 \begin{equation}\label{eq:exp_vi_dual_gap_1stproof} 
    \E{ \sup_{u\in Z}\sum_{\tau=0}^t  \langle \gamma_{\tau}F(\theta_{\tau+1/2}),\theta_{\tau+1/2}-u\rangle } \leq 2\Omega +   \sum_{\tau=0}^t \gamma_{\tau}^2 n \left(\frac{(3-p)G^2}{p} + \sigma^2\right)
 \end{equation}
 Applying \autoref{cneplemma} to equation~\eqref{eq:exp_vi_dual_gap_1stproof} yields an upper bound on the functional Nash error shown in equation~\eqref{boundstoch}.
 \begin{equation} \label{boundstoch}
    \E{ \mathrm{Err}_N(\hat{\theta}_t) } \leq \left(\sum_{\tau = 0}^t \gamma_\tau\right)^{-1}
    \left(2\Omega +   \sum_{\tau=0}^t \gamma_{\tau}^2 n \left(\frac{(3n-b)G^2}{b} + \sigma^2 \right) \right).
 \end{equation}
 
 Now, let us set $\gamma_t$ constant and optimize the bound~\eqref{boundstoch}. Namely, we apply \autoref{equal_terms} setting $\gamma_\tau = 1$ for all $\tau \in [t]$, $A = 2\Omega$ and 
 \begin{equation} B = n\left(\frac{(3n-b)G^2}{b} + \sigma^2 \right). \end{equation}
 
 The optimal value for $\gamma_\tau$ is 
 \begin{equation} \label{eq:gamma_non_smooth}
 \gamma_\tau = \gamma = \sqrt{\frac{2\Omega}{n\left(\frac{(3n-b)G^2}{b}+\sigma^2 \right)t}}.
 \end{equation}
 and the optimal value of the bound is 
 \begin{equation} \label{minconst}
     \E{ \mathrm{Err}_N(\hat{\theta}_t) } \leq \sqrt{\frac{8\Omega n \left(\frac{(3n-b)G^2}{b} + \sigma^2 \right)}{t}}.
 \end{equation}
 
 The number of iterations $t$ can be expressed in terms of the number of gradient computations $k$ as $t(k) = k/(2b)$. Plugging this expression into~\eqref{minconst}, we get
 \begin{equation}
     \E{ \mathrm{Err}_N(\hat{\theta}_{t(k)}) }
     = \sqrt{\frac{8\Omega n \left(\frac{3G^2n}{b} + \sigma^2 - G^2\right)}{\frac{k}{2b}}}, 
 \end{equation}
 which yields equation~\eqref{minconstc} after simplification.
  \end{proof}
 
 \begin{remark}
 For constant stepsizes, equation~\eqref{minconst} implies that with an appropriate choice of $t$ and $\gamma$ we can achieve a value of the Nash error arbitrarily close to zero at time $t$. However, from Equation~\ref{boundstoch} we see that constant stepsizes do not ensure convergence; the bound has a strictly positive limit. Stepsizes decreasing as $1/\sqrt{\tau}$ do ensure convergence, although we do not make a detailed analysis of this case. 
 \end{remark}
 
 \begin{remark}
 Without using any variance reduction technique, the smooth losses assumption \autoref{ass:smooth} does not yield a significant improvement over the bound from \autoref{thmstoch}. We do not include the analysis of this case.
 \end{remark}
 

%% file: app_smooth.tex

\subsection{Doubly-stochastic mirror-prox with variance reduction---Proof of \autoref{thm:VR}}
\label{sec:doub_sto_vr}
\subsubsection{Algorithm}

With the same notations as above, we present a version of~\autoref{alg:doubly_stoch} with variance reduction in the mirror framework.

\begin{algorithm}
\caption{Mirror prox with variance reduced player randomness}
\label{variance_reduction}
\begin{algorithmic}[1] 
\State \textbf{Input}: initial point $\theta_{0}\in \RR^{d}$, stepsizes $(\gamma_{\tau})_{\tau \in [t]}$, mini-batch size over the players $b \in [n]$.
\State Set $R_0 = \hat{F}(\theta_0) \in \mathbb{R}^{d}$
\For {$\tau=0,\dots,t$}
        \State Sample the random matrices $M_{\tau},M_{\tau+1/2} \in \RR^{d\times d}$.
        \State Compute $\tilde{F}_{\tau+1/2} = R_{\tau} + \frac{n}{b}M_{\tau}(\hat{F}(\theta_{\tau})-R_{\tau})$
        \State Set $R_{\tau+1/2} = R_{\tau} + M_{\tau}(\hat{F}(\theta_{\tau})-R_{\tau})$
        \State Extrapolation step:  $\theta_{\tau+1/2} =
         P_{\theta_{\tau}}(\gamma_{\tau}\tilde{F}_{\tau+1/2})$.
        \State Compute $\tilde{F}_{\tau+1} = R_{\tau+1/2} + \frac{n}{b}M_{\tau+1/2}(\hat{F}(\theta_{\tau+1/2})-R_{\tau+1/2})$
        \State Set $R_{\tau+1} = R_{\tau+1/2} + M_{\tau+1/2}(\hat{F}(\theta_{\tau+1/2})-R_{\tau+1/2})$
        \State Extra-gradient step: $\theta_{\tau+1} = 
         P_{\theta_{\tau}}(\gamma_{\tau}\tilde{F}_{\tau+1}).$
\EndFor
\State \textbf{Return} $\hat{\theta}_t = \left[ \sum_{\tau=0}^t \gamma_{\tau}\right]^{-1}\sum_{\tau=0}^t \gamma_{\tau}\theta_{\tau}.$
\end{algorithmic}
\end{algorithm}

$\hat{F}(\theta)$ is defined as in~\autoref{alg:doubly_stoch_mirror}.
The random matrices $M_{\tau},M_{\tau+1/2}$ are also sampled the same way.

In~\autoref{variance_reduction}, we leverage information from a table $(R_\tau)_{\tau \in [t]}$ to produce doubly-stochastic simultaneous gradient estimates with lower variance than in~\autoref{alg:doubly_stoch_mirror}. The table $R_\tau$ is updated when possible. 

The following \autoref{thm:VRmp} generalizes \autoref{thm:VR} in the mirror setting.

\begin{theorem} \label{thm:VRmp}
Assume that for all $i$ between 1 and $n$, the gradients $\nabla_i \ell_i$ are $L$-Lipschitz (\autoref{ass:smooth}). 
Assume \autoref{variance_reduction} is run with constant stepsizes $\gamma_\tau = \gamma$, with $\gamma$ defined as 
\begin{equation} \label{gammadef}
    \gamma \eqdef \min\bigg\{\frac{p^{3/2}}{\sqrt{(1-p)(2-p)}}\frac{1}{12L \sqrt{ n}},\frac{1}{L}\sqrt{\frac{5}{27n+12}}, \frac{1}{2}\sqrt{\frac{\Omega}{13 n\sigma^2 t(k)}}\bigg\},
\end{equation}
where $p \eqdef b/n$, $k$ is the number of gradient computations and $t(k) = k/(2b)$ is the corresponding number of iterations. Then, the convergence rate in expectation at iteration $t(k)$ is 
 \begin{align} 
     \E{ \mathrm{Err}_N(\hat{\theta}_{t(k)}) } 
     &\leq \max \left\{
     \frac{96 \sqrt{2} \Omega L n^2}{\sqrt{b}k}, 8\Omega b L\sqrt{\frac{27n+12}{5}} \frac{1}{k}, 8\sqrt{\frac{26\Omega n b \sigma^2}{k}} \right\}.
 \end{align}
\end{theorem}

\paragraph{Outline of the proof of \autoref{thm:VRmp}.}
\begin{itemize}
    \item \autoref{lemma7} provides a bound for
    $\E{\sum_{\tau=0}^t \gamma_\tau^2 \|\tilde{F}_{\tau+1}-F(\theta_{\tau+1/2})\|_{\star}^2 + \gamma_\tau^2 \|F(\theta_{\tau})-\tilde{F}_{\tau+1/2}\|_{\star}^2}$ and it is the keystone of the proof. It specifically uses the structure of player sampling and the introduced variance reduction mechanism.
    \item \autoref{lemm:equalities} and~\ref{lemm:equalities2} are intermediate steps in the proof of \autoref{lemma7}. \autoref{rm:kj} and \autoref{lemm:sum_deriv} are used in the proof of \autoref{lemm:equalities2}.
    \item We prove \autoref{thm:VRmp} by refining base inequalities established by~\citet{solving}, using the results from \autoref{lemma7}.
\end{itemize}

\begin{definition} \label{def:kj}
For a given $j$ and $i$ (which we omit), let us define $K_{j}$ as the random variable indicating the highest $q \in \mathbb{N}$ strictly lower than $j$ such that $M_{q/2}^{(i)}$ is the identity (and $K_{j} = 0$ if there exists no such $q$).
\end{definition}

In other words, $K_{j}$ is the last step $q$ before $j$ at which the sequence $(R_{q/2}^{(i)})_{q \in \mathbb{N}}$ was updated with a new value $\hat{F}^{(i)}(\theta_{q/2})$. That is, $R_{j/2,i} =\hat{F}^{(i)}(\theta_{K_{j}/2})$. 

\begin{lemma} \label{rm:kj}
For a given $j$, $j - K_{j}$ is a random variable that has a geometric distribution with parameter $p$ and support between 1 and $j$, i.e., for all $q$ such that $j-1 \geq q \geq 1$, 
\begin{equation}
    P(K_{j} = q) = p (1-p)^{j-1-q},
\end{equation}
and $P(K_{j} = 0) = 1 - \sum_{q=1}^{j-1} P(K_j = q) = (1-p)^{j-1}$.
\end{lemma}

\begin{proof}
$M_{q/2}^{(i)}$ is Bernoulli distributed with parameter $p$ among zero and the identity, for all $q$.
\end{proof}

\begin{lemma} \label{lemm:equalities}
  The following equalities hold:
  \begin{align}
    \E{\|F^{(i)}(\theta_{\tau})-\tilde{F}_{\tau+1/2}^{(i)}\|_{\star}^2} &= \frac{2(1-p)}{p}\E{\|R_{\tau}^{(i)}-\hat{F}^{(i)}(\theta_{\tau}) \|_{\star}^2} +2\sigma^2, \label{initial2}\\
      \E{\|\tilde{F}_{\tau+1}^{(i)}-F^{(i)}(\theta_{\tau+1/2})\|_{\star}^2} &=  \frac{2(1-p)}{p}\E{\|R_{\tau+1/2}^{(i)}-\hat{F}^{(i)}(\theta_{\tau+1/2}) \|_{\star}^2} +2\sigma^2.\label{initial3}
 \end{align}
 \end{lemma}

\begin{proof}
Using the conditional expectation with respect to the filtration up to $w_\tau$,
\begin{align}
&\E{\|\tilde{F}_{\tau+1}^{(i)}-F^{(i)}(\theta_{\tau+1/2})\|_{\star}^2}\\
&= 2\E{\bigg\|R_{\tau+1/2}^{(i)} + \frac{M_{\tau+1/2}^{(i)}}{p}(\hat{F}^{(i)}(\theta_{\tau+1/2})-R^{(i)}_{\tau+1/2})-\hat{F}^{(i)}(\theta_{\tau+1/2})\bigg\|_{\star}^2} \\
&+ 2\E{\|\hat{F}^{(i)}(\theta_{\tau+1/2})-F^{(i)}(\theta_{\tau+1/2})\|_{\star}^2}\\
&=
2\E{\bigg\|\left(I-\frac{M_{\tau+1/2}^{(i)}}{p}\right)(R_{\tau+1/2}^{(i)}-\hat{F}^{(i)}(\theta_{\tau+1/2})) \bigg\|_{\star}^2} + 2\sigma^2\\
&= 2\mathbb{E}\left[p \bigg\|\frac{p-1}{p}(R_{\tau+1/2}^{(i)}-\hat{F}^{(i)}(\theta_{\tau+1/2})) \bigg\|_{\star}^2
+ (1-p)\|R_{\tau+1/2}^{(i)}-\hat{F}^{(i)}(\theta_{\tau+1/2}) \|_{\star}^2\right] +2\sigma^2\\ 
&= 2\left(1-p+\frac{(1-p)^2}{p}\right)\E{\|R_{\tau+1/2}^{(i)}-\hat{F}^{(i)}(\theta_{\tau+1/2}) \|_{\star}^2} +2\sigma^2\\
&= \frac{2(1-p)}{p}\E{\|R_{\tau+1/2}^{(i)}-\hat{F}^{(i)}(\theta_{\tau+1/2}) \|_{\star}^2} +2\sigma^2.
\end{align}
The second equality is derived analogously.
\end{proof}

Let us define the change of variables $j = 2\tau$. Parametrized by $j$, the sequences that we are dealing with are $(M_{j/2}^{(i)})_{j \in \mathbb{N}}$, $(R_{j/2}^{(i)})_{j \in \mathbb{N}}$ and $(\theta_{j/2})_{j \in \mathbb{N}}$. In this scope $i$ is a fixed integer between 1 and $n$.





\begin{lemma} \label{lemm:equalities2}
Let us define $h : \mathbb{R} \to \mathbb{R}$ as
\begin{equation} \label{hofp}
    h(p) \eqdef\frac{2-p}{p^2}.
\end{equation}
Assume that $(\gamma_\tau)_{\tau \in \mathbb{N}}$ is non-increasing. Then, the following holds:
\begin{align}  
    &\sum_{\tau = 0}^t \gamma_\tau^2 \E{\|R_{\tau}^{(i)}-\hat{F}^{(i)}(\theta_{\tau})\|_{\star}^2} \leq 
    \sum_{j = 0}^{2t-1} h(p) \gamma_{\floor{j/2}}^2 \E{ \|\hat{F}^{(i)}(\theta_{j/2})-\hat{F}^{(i)}(\theta_{(j+1)/2}) \|_{\star}^2},\label{firstb} \\
    &\sum_{\tau = 0}^t \gamma_\tau^2 \E{\|R_{\tau+1/2}^{(i)}-\hat{F}^{(i)}(\theta_{\tau+1/2}) \|_{\star}^2} \leq 
    \sum_{j = 0}^{2t-1} h(p) \gamma_{\floor{j/2}}^2 \E{ \|\hat{F}^{(i)}(\theta_{j/2})-\hat{F}^{(i)}(\theta_{(j+1)/2}) \|_{\star}^2}.  \label{firstb2}
\end{align}
\end{lemma}

\begin{proof}
We can write
\begin{align} \label{theeq}
    \E{\|R_{\tau}^{(i)}-\hat{F}^{(i)}(\theta_{\tau}) \|_{\star}^2} &= \E{\|R_{2\tau/2}^{(i)}-\hat{F}^{(i)}(\theta_{2\tau/2}) \|_{\star}^2}\\
    &=
    \E{\E{\|R_{2\tau/2}^{(i)}-\hat{F}^{(i)}(\theta_{2\tau/2}) \|_{\star}^2 \bigg|K_{2\tau}}} \\
    &= \sum_{q=0}^{2\tau-1} P(K_{2\tau} = q) \E{ \|R_{2\tau/2}^{(i)}-\hat{F}^{(i)}(\theta_{2\tau/2}) \|_{\star}^2 \bigg|K_{2\tau} = q} \\
    &= \sum_{q=1}^{2\tau-1} p (1-p)^{2\tau-1-q} \E{ \|\hat{F}^{(i)}(\theta_{q/2})-\hat{F}^{(i)}(\theta_{2\tau/2}) \|_{\star}^2}\\
    &+ (1-p)^{2\tau - 1} 
    \E{\|\hat{F}^{(i)}(\theta_0)-\hat{F}^{(i)}(\theta_{2\tau/2})\|_{\star}^2}.
\end{align}
As seen in equation~\eqref{theeq}, the point of conditioning with respect to the sigma-field generated by $K_{2\tau}$ (see \autoref{def:kj}) is that we can write the expression for $R_{2\tau/2,i}$. We have used \autoref{rm:kj}.

Now, using the rearrangement inequality, 
\begin{align} \label{theeq2}
\E{ \|\hat{F}^{(i)}(\theta_{q/2})-\hat{F}^{(i)}(\theta_{2\tau/2}) \|_{\star}^2} &= \E{ \bigg\|\sum_{j=q}^{2\tau-1} \hat{F}^{(i)}(\theta_{j/2})-\hat{F}^{(i)}(\theta_{(j+1)/2}) \bigg\|_{\star}^2} \\ 
&\leq   \sum_{j=q}^{2\tau-1} (2\tau-q) \E{ \|\hat{F}^{(i)}(\theta_{j/2})-\hat{F}^{(i)}(\theta_{(j+1)/2}) \|_{\star}^2}.
\end{align}
Using equations~\eqref{theeq} and~\eqref{theeq2} we can now write 
\begin{align} \label{longbound}
    &\sum_{\tau = 0}^t \gamma_\tau^2 \E{\|R_{\tau}^{(i)}-\hat{F}^{(i)}(\theta_{\tau})\|_{\star}^2}\\
    &= \sum_{\tau = 0}^t \gamma_\tau^2 \sum_{q=1}^{2\tau-1} p (1-p)^{2\tau-1-q} \E{ \|\hat{F}^{(i)}(\theta_{q/2})-\hat{F}^{(i)}(\theta_{2\tau/2}) \|_{\star}^2}\\
    &+ \gamma_\tau^2 (1-p)^{2\tau - 1} 
    \E{\|\hat{F}^{(i)}(\theta_0)-\hat{F}^{(i)}(\theta_{2\tau/2})\|_{\star}^2} \\
    &\leq \sum_{\tau = 0}^t \gamma_\tau^2 \sum_{q=1}^{2\tau-1} p (1-p)^{2\tau-1-q} \sum_{j=q}^{2\tau-1} (2\tau-q) \E{ \|\hat{F}^{(i)}(\theta_{j/2})-\hat{F}^{(i)}(\theta_{(j+1)/2}) \|_{\star}^2}
    \\
    &+ \gamma_\tau^2 (1-p)^{2\tau - 1}\sum_{j=0}^{2\tau-1} 2\tau \E{ \|\hat{F}^{(i)}(\theta_{j/2})-\hat{F}^{(i)}(\theta_{(j+1)/2}) \|_{\star}^2}.
\end{align}
Given $j$ between 0 and $2t-1$ the right hand side of equation~\eqref{longbound} contains the term \linebreak $\E{ \|\hat{F}^{(i)}(\theta_{j/2})-\hat{F}^{(i)}(\theta_{(j+1)/2}) \|_2^2}$ multiplied by 
\begin{align} \label{factor}
    &\sum_{\tau=\ceil{(j+1)/2}}^t \gamma_\tau^2 \left(\sum_{r=1}^{j}(2\tau-r) p(1-p)^{2\tau-1-r} + 2\tau(1-p)^{2\tau-1} \right)\\ 
    &\leq 
    \gamma_{\floor{j/2}}^2 \sum_{\tau=\ceil{(j+1)/2}}^t \sum_{r=1}^{j}(2\tau-r) p(1-p)^{2\tau-1-r} + 2\tau(1-p)^{2\tau-1}
    \\
    &= \gamma_{\floor{j/2}}^2 \sum_{\tau=\ceil{(j+1)/2}}^t p \sum_{r'=0}^{j-1} (1-p)^{2\tau-1-j+r'} (2\tau-j+r') + 2\tau(1-p)^{2\tau-1}
    \\ 
    &\leq \gamma_{\floor{j/2}}^2 \sum_{\tau=\ceil{(j+1)/2}}^t p 
    \sum_{r'=2\tau-j}^{\infty} (1-p)^{r'-1} r' = (*).
\end{align}
Using \autoref{lemm:sum_deriv} twice:
\begin{align}
    (*) &= \gamma_{\floor{j/2}}^2 \sum_{\tau=\ceil{(j+1)/2}}^t p
    \frac{(2\tau-j) (1-p)^{2\tau-1-j} p+(1-p)^{2\tau-j}}{p^2} \\
    &= \gamma_{\floor{j/2}}^2 \sum_{\tau=\ceil{(j+1)/2}}^t 
    \frac{(2\tau-j) (1-p)^{2\tau-1-j} p+(1-p)^{2\tau-j}}{p} 
    \\ 
    &\leq \gamma_{\floor{j/2}}^2 \sum_{\tau=2\ceil{(j+1)/2}}^\infty (\tau -j)(1-p)^{\tau-1-j} + \frac{\gamma_{\floor{j/2}}^2}{p} \sum_{\tau=2\ceil{(j+1)/2}}^\infty (1-p)^{\tau-j} \\
    &=  \gamma_{\floor{j/2}}^2 \sum_{\tau=2\ceil{(j+1)/2}-j}^\infty \tau(1-p)^{\tau-1} + \frac{\gamma_{\floor{j/2}}^2}{p} \sum_{\tau=2\ceil{(j+1)/2}-j}^\infty (1-p)^{\tau} \\
    &=\gamma_{\floor{j/2}}^2 \frac{(2\ceil{(j+1)/2}-j) (1-p)^{2\ceil{(j+1)/2}-j-1} p+2(1-p)^{2\ceil{(j+1)/2}-j}}{p^2}.
\end{align}
By \autoref{longineq} we have
\begin{equation} \label{hdef}
    \frac{(2\ceil{(j+1)/2}-j) (1-p)^{2\ceil{(j+1)/2}-j-1} p+2(1-p)^{2\ceil{(j+1)/2}-j}}{p^2}.
    \leq h(p)
\end{equation}
Hence, from equation~\eqref{longbound} we get
\begin{equation} 
    \sum_{\tau = 0}^t \gamma_\tau^2 \E{\|R_{\tau}^{(i)}-\hat{F}^{(i)}(\theta_{\tau})\|_{\star}^2} \leq 
    \sum_{j = 0}^{2t-1} \gamma_{\floor{j/2}}^2 h(p) \E{ \|\hat{F}^{(i)}(\theta_{j/2})-\hat{F}^{(i)}(\theta_{(j+1)/2}) \|_{\star}^2}.
\end{equation}
Analogously to equation~\eqref{theeq}:
\begin{align} \label{theeq43}
    &\E{\|R_{\tau+1/2}^{(i)}-\hat{F}^{(i)}(\theta_{\tau+1/2}) \|_{\star}^2}\\
    &= \E{\|R_{(2\tau+1)/2}^{(i)}-\hat{F}^{(i)}(\theta_{(2\tau+1)/2}) \|_{\star}^2} \\
    &= \E{\E{\|R_{(2\tau+1)/2}^{(i)}-\hat{F}^{(i)}(\theta_{(2\tau+1)/2}) \|_{\star}^2 \bigg|K_{2\tau+1}}} \\
    &= \sum_{k=0}^{2\tau} P(K_{2\tau+1} = k) \E{ \|R_{(2\tau+1)/2}^{(i)}-\hat{F}^{(i)}(\theta_{(2\tau+1)/2}) \|_{\star}^2 \bigg|K_{2\tau+1} = k} \\
    &= \sum_{k=1}^{2\tau} p (1-p)^{2\tau-k} \E{ \|\hat{F}^{(i)}(\theta_{k/2})-\hat{F}^{(i)}(\theta_{(2\tau+1)/2}) \|_{\star}^2}\\
    &+ (1-p)^{2\tau} \E{ \|\hat{F}^{(i)}(\theta_0)-\hat{F}^{(i)}(\theta_{(2\tau+1)/2})\|_{\star}^2}.
\end{align}
Using the same reasoning we get an inequality that is analogous to~\eqref{firstb}:
\begin{equation*} 
    \sum_{\tau = 0}^t \gamma_\tau^2 \E{\|R_{\tau+1/2}^{(i)}-\hat{F}^{(i)}(\theta_{\tau+1/2}) \|_{\star}^2} \leq 
    \sum_{j = 0}^{2t} \gamma_{\floor{j/2}}^2 h(p) \E{ \|\hat{F}^{(i)}(\theta_{j/2})-\hat{F}^{(i)}(\theta_{(j+1)/2}) \|_{\star}^2}.\qedhere
\end{equation*}
\end{proof}

\begin{lemma} \label{lemma7}
Assume that for all $i$ between 1 and $n$, the gradients $\nabla_i \ell_i$ are $L$-Lipschitz.
Assume that for all $\tau$ between 0 and $t$, $\gamma_{\tau} \leq \gamma$. Let 
\begin{equation} \label{chidef}
    \chi(p,\gamma) = 1-36\frac{1-p}{p} n h(p) L^2 \gamma^2 .
\end{equation}
If $\gamma$ is small enough that $\chi(p,\gamma)$ is positive, then
\begin{align} \label{lemma7eq}
    &\E{\sum_{\tau=0}^t \gamma_\tau^2 \|\tilde{F}_{\tau+1}-F(\theta_{\tau+1/2})\|_{\star}^2 + \gamma_\tau^2 \|F(\theta_{\tau})-\tilde{F}_{\tau+1/2}\|_{\star}^2} \\ 
    &\leq 104n\sigma^2\sum_{\tau = 0}^{t} \gamma_\tau^2 + \frac{1-p}{p\chi(p,\gamma)}
    (12L^2+ 36 L^4 \gamma^2) n h(p) \sum_{\tau = 0}^{t} \gamma_\tau^2 \E{ \|\theta_{\tau}-\theta_{\tau+1/2}\|_{\star}^2}.
\end{align}
\end{lemma}

\begin{proof}
We first want to bound the terms $\E{ \|F^{(i)}(\theta_{j/2})-F^{(i)}(\theta_{(j+1)/2}) \|_2^2}$. When $j$ is even we can make the change of variables $j/2 = \tau$ (just for simplicity in the notation) and use smoothness. We get
\begin{align} \label{small1}
    \E{ \|\hat{F}^{(i)}(\theta_{j/2})-\hat{F}^{(i)}(\theta_{(j+1)/2}) \|_{\star}^2} &= \E{ \|\hat{F}^{(i)}(\theta_{\tau})-\hat{F}^{(i)}(\theta_{\tau+1/2}) \|_{\star}^2} \\
    &\leq 3\E{ \|F^{(i)}(\theta_{\tau})-F^{(i)}(\theta_{\tau+1/2}) \|_{\star}^2}\\
    &+ 3\E{ \|\hat{F}^{(i)}(\theta_{\tau})-F^{(i)}(\theta_{\tau+1/2}) \|_{\star}^2} \\
    &+ 3\E{ \|\hat{F}^{(i)}(\theta_{\tau+1/2})-F^{(i)}(\theta_{\tau+1/2}) \|_{\star}^2}\\
    &\leq 3L^2 \E{ \|\theta_{\tau}-\theta_{\tau+1/2} \|_{\star}^2} +6\sigma^2.
\end{align}
When $j$ is odd, we can write $j/2 = \tau+1/2$. We use smoothness and the fact that the prox-mapping is 1-Lipschitz (Lemma \ref{lem:lemma4}):
\begin{align} \label{small2}
    \E{ \|\hat{F}^{(i)}(\theta_{j/2})-\hat{F}^{(i)}(\theta_{(j+1)/2}) \|_{\star}^2}
    &= \E{ \|\hat{F}^{(i)}(\theta_{\tau+1/2})-\hat{F}^{(i)}(\theta_{\tau+1}) \|_{\star}^2} \\
    &\leq 3\E{ \|F^{(i)}(\theta_{\tau+1/2})-F^{(i)}(\theta_{\tau+1}) \|_{\star}^2}\\
    &+ 3\E{ \|\hat{F}^{(i)}(\theta_{\tau+1/2})-F^{(i)}(\theta_{\tau+1/2}) \|_{\star}^2} \\
    &+3\E{ \|\hat{F}^{(i)}(\theta_{\tau+1})-F^{(i)}(\theta_{\tau+1}) \|_{\star}^2}\\
    &\leq 3L^2 \E{\|\theta_{\tau+1/2} - \theta_{\tau+1}\|_{\star}^2} + 6\sigma^2 \\
    &= 3L^2 \E{\|P_{\theta_{\tau}}(\gamma_{\tau} \tilde{F}_{\tau+1/2}) - P_{\theta_{\tau}}(\gamma_{\tau} \tilde{F}_{\tau+1})\|_{\star}^2} + 6\sigma^2 \\
    &\leq 3L^2 \gamma_{\tau}^2 \E{\|\tilde{F}_{\tau+1/2} -\tilde{F}_{\tau+1}\|_{\star}^2} + 6\sigma^2 \\ 
    &\leq  9 L^2 \gamma_{\tau}^2\left(\E{\|\tilde{F}_{\tau+1/2}-F(\theta_{\tau})\|_{\star}^2} \right. \\
     & \left.\qquad +\E{\|F(\theta_{\tau+1/2})-\tilde{F}_{\tau+1}\|_{\star}^2}\right.\\
     &\left. \qquad + \E{\|F(\theta_{\tau})-F(\theta_{\tau+1/2})\|_{\star}^2}\right) + 6\sigma^2.
\end{align}

Now, we use \autoref{dual} to break up the dual norms in the right-hand side of \eqref{lemma7eq}.
 \begin{align} 
     &\E{\sum_{\tau=0}^t \gamma_\tau^2\|\tilde{F}_{\tau+1}-F(\theta_{\tau+1/2})\|_{\star}^2 + \gamma_\tau^2\|F(\theta_{\tau})-\tilde{F}_{\tau+1/2}\|_{\star}^2} \\ 
     &= \E{\sum_{\tau=0}^t \sum_{i=1}^n \gamma_\tau^2 \|\tilde{F}_{\tau+1}^{(i)}-F^{(i)}(\theta_{\tau+1/2})\|_{\star}^2 + \gamma_\tau^2 \|F^{(i)}(\theta_{\tau})-\tilde{F}_{\tau+1/2}^{(i)}\|_{\star}^2},\label{initial1}
 \end{align}

Hence, from equation~\eqref{initial1} and \autoref{lemm:equalities} and~\ref{lemm:equalities2}: 
\begin{align}
    &\E{\sum_{\tau=0}^t \gamma_\tau^2 \|\tilde{F}_{\tau+1}-F(\theta_{\tau+1/2})\|_{\star}^2 + \gamma_\tau^2 \|F(\theta_{\tau})-\tilde{F}_{\tau+1/2}\|_{\star}^2} 
    \\ &\leq 4n\sigma^2\sum_{\tau = 0}^{t} \gamma_\tau^2 + \frac{2(1-p)}{p}\E{\sum_{\tau=0}^t \sum_{i=1}^n \gamma_t^2 \|R_\tau^{(i)} - \hat{F}^{(i)}(\theta_\tau) \|^2 + \|R_{\tau+1/2}^{(i)} - \hat{F}^{(i)}(\theta_\tau)\|^2}
    \\ &\leq 4n\sigma^2\sum_{\tau = 0}^{t} \gamma_\tau^2 + \frac{2(1-p)}{p}
    \sum_{i=1}^n \sum_{j = 0}^{2t} 2 \gamma_{\floor{j/2}}^2 h(p) \E{ \|F_i(\theta_{j/2})-F_i(\theta_{(j+1)/2}) \|_{\star}^2} = (**).
\end{align}
We split the last term in summands corresponding to even and odd $j$, we change variables from $j$ to $\tau$ and we apply equations~\eqref{small1} and~\eqref{small2}:
\begin{align}
    (**) &= 4n\sigma^2\sum_{\tau = 0}^{t} \gamma_\tau^2 +
    \frac{2(1-p)}{p}\sum_{i=1}^n \sum_{j = 0, \ j \text{ even}}^{2t} 2 \gamma_{\floor{j/2}}^2 h(p) \E{ \|F_i(\theta_{j/2})-F_i(\theta_{(j+1)/2}) \|_{\star}^2} \\ &+\frac{2(1-p)}{p} \sum_{i=1}^n \sum_{j = 0, \ j \text{ odd}}^{2t} 2 \gamma_{\floor{j/2}}^2 h(p) \E{ \|F_i(\theta_{j/2})-F_i(\theta_{(j+1)/2}) \|_{\star}^2} \\ 
    &= 4n\sigma^2\sum_{\tau = 0}^{t} \gamma_\tau^2 + \frac{2(1-p)}{p}\sum_{i=1}^n \sum_{\tau = 0}^{t} 2 \gamma_\tau^2 h(p) \E{ \|F_i(\theta_\tau)-F_i(\theta_{\tau+1/2}) \|_{\star}^2} \\
    &+ \frac{2(1-p)}{p} \sum_{i=1}^n \sum_{\tau = 0}^{t} 2 \gamma_\tau^2 h(p) \E{ \|F_i(\theta_{\tau+1/2})-F_i(\theta_{\tau+1}) \|_{\star}^2} \\
    &\leq 52n\sigma^2\sum_{\tau = 0}^{t} \gamma_\tau^2 +
    \frac{1-p}{p} \sum_{\tau = 0}^{t} 12n \gamma_\tau^2 h(p) L^2 \E{ \|\theta_{\tau}-\theta_{\tau+1/2} \|_{\star}^2}\\ &+\frac{1-p}{p} \sum_{\tau = 0}^{t} 36n h(p) L^2 \gamma_\tau^4 \left(\E{\|\tilde{F}_{\tau+1/2}-F(\theta_{\tau})\|_{\star}^2} +\E{\|F(\theta_{\tau+1/2})-\tilde{F}_{\tau+1}\|_{\star}^2} 
      \right) \\
    &  +\frac{1-p}{p} \sum_{\tau = 0}^{t} 36n h(p) L^4 \gamma_{\tau}^4 \E{ \|\theta_{\tau}-\theta_{\tau+1/2} \|_{\star}^2} = (***).
\end{align}
We use that $\gamma_\tau \leq \gamma$:
\begin{align}
    (***) &\leq
      52n\sigma^2\sum_{\tau = 0}^{t} \gamma_\tau^2 + \frac{1-p}{p}
    (12L^2+ 36 L^4 \gamma^2) n h(p) \sum_{\tau = 0}^{t} \gamma_\tau^2 \E{ \|\theta_{\tau}-\theta_{\tau+1/2}\|_{\star}^2}\\ 
    &+36 \frac{1-p}{p} n h(p) L^2 \gamma^2 \sum_{\tau = 0}^{t} \gamma_\tau^2 \left(\E{\|\tilde{F}_{\tau+1/2}-F(\theta_{\tau})\|_{\star}^2} +\E{\|F(\theta_{\tau+1/2})-\tilde{F}_{\tau+1}\|_{\star}^2} 
      \right). 
\end{align}
Rearranging and using $\chi(p,\gamma) > 0$ yields the desired result.
\end{proof}

\begin{proof}[Proof of \autoref{thm:VRmp}]

We rewrite equation~\eqref{basiceq}:
\begin{align}
     \label{eq:lemma1_application2}
          &\langle \gamma_{\tau}\tilde{F}_{\tau+1},\theta_{\tau+1/2}-u\rangle + D(u,\theta_{\tau+1})-D(u,\theta_{\tau})\\ 
          &\leq \frac{\gamma_{\tau}^2}{2}\|\tilde{F}_{\tau+1}-\tilde{F}_{\tau+1/2}\|_{\star}^2-\frac{1}{2}\|\theta_{\tau+1/2}-\theta_{\tau}\|^2\\
     &\leq \frac{3\gamma_{\tau}^2}{2}\|\tilde{F}_{\tau+1}-F(\theta_{\tau+1/2})\|_{\star}^2 +\frac{3\gamma_{\tau}^2}{2}\|F(\theta_{\tau})-\tilde{F}_{\tau+1/2}\|_{\star}^2 
      + \frac{3\gamma_{\tau}^2}{2}\|F(\theta_{\tau+1/2})-F(\theta_{\tau})\|_{\star}^2 \\
     &-\frac{1}{2}\|\theta_{\tau+1/2}-\theta_{\tau}\|^2 .
\end{align}
 We rewrite equation~\eqref{eq:result_lemma23}. We have $\Delta_{\tau} = F(\theta_{\tau+1/2})-\tilde{F}_{\tau+1}$ and $y_{\tau+1} = P_{y_{\tau}}(\gamma_\tau \Delta_{\tau})$ with $y_0 = \theta_0$.
 \begin{align}
    \sum_{\tau=0}^t \langle \gamma_{\tau}\Delta_{\tau},y_{\tau}-u\rangle &\leq D(u,\theta_0) + \sum_{\tau=0}^t \frac{\gamma_{\tau}^2}{2}\| \Delta_{\tau}\|_{\star}^2 \\
    &= D(u,\theta_0) +  \sum_{\tau=0}^t \frac{\gamma_{\tau}^2}{2}\| F(\theta_{\tau+1/2})-\tilde{F}_{\tau+1}\|_{\star}^2. \label{eq:result_lemma2_23}
\end{align}
 Using equation~\eqref{eq:result_lemma2_23} and the analogous equation to~\eqref{eq:unbiased_dotprod_1stproof}, we reach the following inequality:

\begin{align} \label{longeq_32}
      &\E{\sup_{u \in Z} \sum_{\tau=0}^t  \langle \gamma_{\tau}F(\theta_{\tau+1/2}),\theta_{\tau+1/2}-u\rangle} \leq 
         \E{\sup_{u \in Z} 2D(u,\theta_{0})-D(u,\theta_{t+1}) - \sum_{\tau=0}^t\frac{1}{2}\|\theta_{\tau+1/2}-\theta_{\tau}\|_2^2}\\
         &+ \mathbb{E}\left[\sum_{\tau=0}^t 2\gamma_{\tau}^2\|\tilde{F}_{\tau+1}-F(\theta_{\tau+1/2})\|_{\star}^2
         +\frac{3\gamma_{\tau}^2}{2}\|F(\theta_{\tau})-\tilde{F}_{\tau+1/2}\|_{\star}^2 
         + \frac{3\gamma_{\tau}^2}{2}\|F(\theta_{\tau+1/2})-F(\theta_{\tau})\|_{\star}^2\right]
 \end{align}

Taking the definition of $\chi(p,\gamma)$ in \eqref{chidef}, using the definition of $h(p)$ in \eqref{hofp} and rearranging, we obtain
\begin{equation} \label{eq:gamma_constraint1}
    \gamma \leq \frac{p^{3/2}}{\sqrt{(1-p)(2-p)}}\frac{1}{12L \sqrt{n}} \iff \chi(p,\gamma) \geq 3/4 > 0.
\end{equation}
Hence, the assumptions of \autoref{lemma7} are fulfilled.
Starting from the result in \eqref{longeq_32} and using \autoref{lemma7},
\begin{align} \label{longeq_323}
      &\E{\sup_{u \in Z} \sum_{\tau=0}^t  \langle \gamma_{\tau}F(\theta_{\tau+1/2}),\theta_{\tau+1/2}-u\rangle} \\
      &\leq 
         \E{\sup_{u \in Z} 2D(u,\theta_{0})-D(u,\theta_{t})} + 32 n \sigma^2 \sum_{\tau=0}^t \gamma_\tau^2 \\
         &+2 \frac{1-p}{p\chi(p,\gamma)}
         (12L^2+ 36 L^4 \gamma^2) n h(p) \sum_{\tau = 0}^{t} \gamma^2_\tau \E{ \|\theta_{\tau}-\theta_{\tau+1/2} \|_{\star}^2}
         \\
         &+ \frac{3 n L^2}{2} \sum_{\tau = 0}^{t} \gamma^2_\tau \E{ \|\theta_{\tau}-\theta_{\tau+1/2} \|_{\star}^2} - \frac{1}{2}\sum_{\tau = 0}^{t}\E{ \|\theta_{\tau}-\theta_{\tau+1/2} \|_{\star}^2} 
         \\ 
         &\leq 2 \Omega + 104 n \sigma^2 \sum_{\tau=0}^t \gamma_\tau^2\\
         &+ \left((24L^2+ 72 L^4 \gamma^2) n h(p) \gamma^2 \frac{1-p}{p\chi(p,\gamma)} + \frac{3n\gamma^2 L^2}{2} - \frac{1}{2}\right)\sum_{\tau = 0}^{t}\E{ \|\theta_{\tau}-\theta_{\tau+1/2} \|_{\star}^2}.
\end{align}
Recalling the definition of $h(p)$ in Equation~\eqref{hofp}, the conditions $\chi(p,\gamma) \geq 3/4$ and 
\begin{equation} \label{eq:gamma_constraint2} \gamma \leq \frac{1}{L}\sqrt{\frac{5}{27n+12}}, 
\end{equation}
imply
\begin{equation} \label{lastcond}
    (24L^2+ 72 L^4 \gamma^2) n h(p) \gamma^2 \frac{1-p}{p\chi(p,\gamma)} + \frac{3n\gamma^2 L^2}{2} - \frac{1}{2} \leq 0.
\end{equation}
We show this development:
\begin{align}
    &(24L^2+ 72 L^4 \gamma^2) n \frac{2-p}{p^2} \gamma^2 \frac{1-p}{p\chi(p,\gamma)} + \frac{3n\gamma^2 L^2}{2} - \frac{1}{2} \\ 
    &\stackrel{\chi \geq 3/4}{\leq} (24L^2+ 72 L^4 \gamma^2) n \frac{2-p}{p^2} \gamma^2 \frac{4(1-p)}{3p} + \frac{3n\gamma^2 L^2}{2} - \frac{1}{2}\\
    &= \frac{24+72L^2 \gamma^2}{27} (1-\chi(p,\gamma))
    + \frac{3n\gamma^2 L^2}{2} - \frac{1}{2} \\ 
    &\leq
    \frac{2+6L^2\gamma^2}{9} + \frac{3n\gamma^2 L^2}{2} - \frac{1}{2} \\
    &= \gamma^2 \frac{(9n+4)L^2}{6} - \frac{5}{18}. 
\end{align}
Using Equation~\eqref{lastcond} on \eqref{longeq_323} yields
\begin{equation}
    \E{\sup_{u \in Z} \sum_{\tau=1}^t  \langle \gamma_{\tau}F(\theta_{\tau+1/2}),\theta_{\tau+1/2}-u\rangle} \leq 2 \Omega + 104 n \sigma^2 \sum_{\tau=0}^t \gamma_\tau^2.
\end{equation}
By \autoref{cneplemma}, we conclude 
\begin{equation} \label{eq:vrfirst}
\mathrm{Err}_N(\hat{\theta}_t) \leq
\left(\sum_{\tau = 0}^t \gamma_\tau\right)^{-1} \left(
    2 \Omega + 104 n \sigma^2 \sum_{\tau=0}^t \gamma_\tau^2
    \right)
 \end{equation}
 
 Now we apply \autoref{equal_terms} to equation~\eqref{eq:vrfirst} assuming constant stepsizes. That is, we set $\gamma_\tau = 1$, $A = 2\Omega$ and $B = 104 n \sigma^2$. Using the notation from \autoref{equal_terms}, we get that 
 \begin{equation} \alpha^{*} = \frac{1}{2}\sqrt{\frac{\Omega}{13 n\sigma^2 t}} \end{equation} and the value of the bound at $\alpha^{*}$ is
 \begin{equation} 8\sqrt{\frac{13 \Omega n\sigma^2}{t}}. \end{equation}
 However, $\gamma$ is also subject to the constraints in equations~\eqref{eq:gamma_constraint1} and \eqref{eq:gamma_constraint2}. Namely,
 \begin{equation} \label{eq:gammadef2}
    \gamma \eqdef \min\bigg\{\frac{p^{3/2}}{\sqrt{(1-p)(2-p)}}\frac{1}{12L \sqrt{ n}},\frac{1}{L}\sqrt{\frac{5}{27n+12}}, \frac{1}{2}\sqrt{\frac{\Omega}{13 n\sigma^2 t}}\bigg\},
\end{equation}
If the minimum in equation~\eqref{eq:gammadef2} is not achieved at $\alpha^{*}$ (the third term), it is easy to see that the first term of the bound in equation~\eqref{eq:vrfirst} is larger than the second one, which means that $4\Omega/(\gamma t)$ is a looser bound. We conclude   
 \begin{equation} \label{eq:vrresult}
     \E{ \mathrm{Err}_N(\hat{\theta}_t) } \leq \max \left\{\frac{4\Omega}{\gamma t}, 8\sqrt{\frac{13 \Omega n\sigma^2}{t}} \right\}.
 \end{equation}
 Substituting $\gamma$ for its expression and plugging $t(k) = k/2b$ on equation~\ref{eq:vrresult} we get
 \begin{equation}
     \E{ \mathrm{Err}_N(\hat{\theta}_{t(k)}) } \leq \max \left\{
     \frac{4\Omega}{\frac{\left(\frac{b}{n}\right)^{3/2}}{\sqrt{(1-\frac{b}{n})(2-\frac{b}{n})}}\frac{1}{12L \sqrt{ n}} \frac{k}{2b}}, \frac{4\Omega}{\frac{1}{L}\sqrt{\frac{5}{27n+12}} \frac{k}{2b}}, 8\sqrt{\frac{26\Omega n b \sigma^2}{k}} \right\}. 
 \end{equation}
 The result follows using $1-b/n < 1$ and $2-b/n < 2$.
 \end{proof}

%% file: appendix_radius.tex

\clearpage

\section{Spectral convergence analysis for non-constrained 2-player games} \label{sec:spectral_conv}
We observed in the experimental section that player sampling tended to be empirically faster than full extra-gradient, and that cyclic sampling had a tendency to be better than random sampling.

To have more insight on this finding, let us study a simplified version of the random two-player quadratic games. Let $A \in \mathbb{R}^{2d\times 2d}$ be formed by stacking the matrices $A_i \in \mathbb{R}^{d \times 2d}$ for each $i\in [d]$. We assume that $A$ is invertible and has a positive semidefinite symmetric part. For $i\in \{1,2\}$, we define the loss of the $i$-th player $\ell_i$ as
\begin{equation}
    \ell_i(\theta^{i}, \theta^{-i}) = {\theta^{i}}^\top A_i \theta - \frac{1}{2} {\theta^{i}}^\top A_{ii} \theta^i, 
\end{equation}
where $A_{ii} \in \mathbb{R}^d$ and $\theta_i \in \mathbb{R}^{d_i}$. Contrary to the random quadratic games setting in~\autoref{sec:quadratic}, we do not enforce here any parameter constraints nor regularization. Therefore, this places us in the extra-gradient (Euclidean) setting. We restrict our attention to the non-noisy regime.

\subsection{Recursion operator for the different sampling schemes}

We study the ``algorithm operator'' $\Aa$ that appears in the recursion $\theta_{k + 4} = \Aa(\theta_{k})$ for the different sampling schemes. $k$ is the number of gradient computations. We consider steps of $4$ evaluation as this corresponds to a single iteration of full extra-gradient.

\paragraph{Full extrapolation and update.} We have $\nabla_i \ell_i(\theta) = A_i \theta$. Since $A$ is invertible, $\theta = 0$ is the only Nash equilibrium. 
The full extra-gradient updates with constant stepsize are
\begin{equation}\label{eq:full_extra_grad_updates}
    \begin{split}
       \begin{cases}
            \theta_{k + 2}^{\full} = \theta_{k}^{\full} - \gamma A \theta_{k}^{\full}, \\
            \theta_{k + 4}^{\full} = \theta_{k }^{\full} - \gamma A \theta_{k + 2}^{\full}.
        \end{cases} 
    \end{split}
\end{equation}
By introducing $\mathcal{A}_{\full}^{(\gamma)} := I - \gamma A + \gamma^2 A^2$, ~\eqref{eq:full_extra_grad_updates} is simply $\theta_{k + 4}^{\full} = \mathcal{A}_{\full}^{(\gamma)} \theta_k^{\full}.$

\paragraph{Cyclic sampling.} Defining the matrices $M_1,M_2 \in \mathbb{R}^{2d \times 2d}$
\begin{equation}
    M_1 = \begin{bmatrix}
        I_{d} & 0_{d \times d} \\
        0_{d\times d} & 0_{d \times d}
    \end{bmatrix}, \quad 
    M_2 = \begin{bmatrix}
        0_{d \times d} & 0_{d \times d} \\
        0_{d\times d} & I_{d}
    \end{bmatrix},
\end{equation}
the updates becomes
\begin{equation}
    \begin{split}
       \begin{cases}
            \theta_{k +1}^{\cyc} = \theta_{k}^{\cyc} - \gamma M_1 A \theta_{k}^{\cyc}, \\
            \theta_{k +2}^{\cyc} = \theta_{k }^{\cyc} - \gamma M_2 A \theta_{k +1}^{\cyc}, \\
            \theta_{k +3}^{\cyc} = \theta_{k+2}^{\cyc} - \gamma M_2 A \theta_{k+2}^{\cyc}, \\
            \theta_{k +4}^{\cyc} = \theta_{k + 2}^{\cyc} - \gamma M_1 A \theta_{k + 3}^{\cyc}. \label{eq:alt_extra}
       \end{cases}.
    \end{split}
\end{equation}
Remark that \eqref{eq:alt_extra} contains two iterations of \autoref{alg:doubly_stoch}; $\theta_{k +1}$ and $\theta_{k +3}$ are extrapolations and $\theta_{k +2}$ and $\theta_{k +4}$ are updates.
Defining
$\mathcal{A}_{ij}^{(\gamma)} := I - \gamma M_i A + \gamma^2 M_i A M_j A$
and $\mathcal{A}_{\text{\cyc}}^{(\gamma)} := \mathcal{A}_{12}^\gamma \mathcal{A}_{21}^{(\gamma)}$, we have $\theta_{k + 4}^{\text{\cyc}} = \mathcal{A}_{\text{\cyc}}^{(\gamma)} \theta_{k}^{\text{\cyc}}$.

\paragraph{Random sampling.} Extra-gradient with random subsampling ($b=1$) rewrites as
\begin{equation}
    \begin{split}
        \begin{cases}
             \theta_{k +1}^{\rdm} = \theta_{k}^{\rdm} - \gamma M_{S_{k+1}} A \theta_{k}^{\rdm}, \\
            \theta_{k +2}^{\rdm} = \theta_{k }^{\rdm} - \gamma M_{S_{k+2}} A \theta_{k +1}^{\rdm}, \\
            \theta_{k +3}^{\rdm} = \theta_{k+2}^{\rdm} - \gamma M_{S_{k+3}} A \theta_{k+2}^{\rdm}, \\
            \theta_{k +4}^{\rdm} = \theta_{k+2}^{\rdm} - \gamma M_{S_{k+3}} A \theta_{k +3}^{\rdm}.
        \end{cases}
    \end{split}.
\end{equation}
where $S_{k+1}, S_{k+2}, S_{k+3}, S_{k+4}$ take values 1 and 2 with equal probability and pairwise are independent. Note that we also enroll two iterations of sampled extra-gradient, as we consider a budget of 4 gradient evaluations. Let $\mathcal{F}_k = \sigma(S_{k'}: k' \leq k)$. For extra-gradient with random player sampling, we can write
\begin{align}
    \E{\theta_{k+4}^{\rdm}} &= \E{\mathcal{A}_{S_{k+1}S_{k+3}}^{(\gamma)} \mathcal{A}_{S_{k+2}S_{k+1}}^{(\gamma)} \theta_{k}^{\rdm}}\\
    &= \E{\E{\mathcal{A}_{S_{k+1}S_{k+3}}^{(\gamma)} \mathcal{A}_{S_{k+2}S_{k+1}}^{(\gamma)} \theta_{k}^{\rdm}\bigg| \mathcal{F}_k}} \\ &= \E{\E{\mathcal{A}_{S_{k+1}S_{k+3}}^{(\gamma)} \mathcal{A}_{S_{k+2}S_{k+1}}^{(\gamma)} \bigg| \mathcal{F}_k} \theta_{k}^{\rdm}}\\
    &= \E{\mathcal{A}_{S_{k+4}S_{k+3}}^{(\gamma)} \mathcal{A}_{S_{k+2}S_{k+1}}^{(\gamma)}} \E{\theta_{k}^{\rdm}}\\
    &=\frac{1}{16} \sum_{j_1, j_2, j_3, j_4 \in \{1,2\}} \mathcal{A}_{j_1 j_2}^{(\gamma)} \mathcal{A}_{j_3 j_4}^{(\gamma)} \E{\theta_{k}^{\rdm}} \\
    &= \frac{1}{16} \left(4I - 2 \gamma A + \gamma^2 A^2 \right)^2 \E{\theta_{k}^{\rdm}} \triangleq \mathcal{A}_{^{\rdm}}^{(\gamma)} \E{\theta_{k}^{\rdm}}
\end{align}

\subsection{Convergence behavior through spectral analysis}

The following well-known result proved by \citet{gelfand1941normierte} relates matrix norms with spectral radii.
\begin{theorem}[Gelfand's formula]
    Let $\|\cdot\|$ be a matrix norm on $\mathbb{R}^n$ and let $\rho(A)$ be the spectral radius of $A \in \mathbb{R}^n$ (the maximum absolute value of the eigenvalues of $A$). Then,
    \begin{equation}
        \lim_{t\to \infty} \|A^t\|^{1/t} = \rho(A).
    \end{equation}
\end{theorem}

In our case, we thus have the following results, that describes the expected rate of convergence of the last iterate sequence ${(\theta_t)}_t$ towards~$0$. It is governed by the spectral radii $\rho(\Aa^{(\eta)})$ whenever the later is strictly lower than $1$.

\begin{corollary} \label{cor:spectral}
The behavior of $\theta_t^{\full}$, $\theta_t^{\cyc}$ and $\theta_t^{\rdm}$ is related to the corresponding operators by the following expressions:
\begin{align}
    \lim_{t \to \infty} \left( \sup_{\theta_0^{\full} \in \mathbb{R}^{2d}} \frac{{\|\theta_t^{\full} \|}_2}{{\|\theta_0^{\full} \|}_2} \right)^{1/t} &= \rho\left(\mathcal{A}_{\full}^{(\gamma)}\right), \\
    \lim_{t\to \infty} \left( \sup_{\theta_0^{\cyc} \in \mathbb{R}^{2d}} \frac{{\|\theta_{t}^{\cyc} \|}_2}{{\|\theta_0^{\cyc} \|}_2} \right)^{1/t} &=
    \rho\left(\mathcal{A}_{\cyc}^{(\gamma)}\right), \\
    \lim_{t \to \infty} \left( \sup_{\theta_0^{\rdm} \in \mathbb{R}^{2d}} \frac{{\|\E{\theta_{t}^{\rdm}} \|}_2}{{\|\theta_0^{\rdm} \|}_2} \right)^{1/t} &=
    \rho\left(\mathcal{A}_{\rdm}^{(\gamma)}\right).
\end{align}
\end{corollary}

\begin{proof}
The proof is analogous for the three cases. Using the definition of operator norm,
\begin{equation}
    \lim_{t \to \infty} \left( \sup_{\theta_0^{\full} \in \mathbb{R}^{2d}} \frac{\|\theta_t^{\full} \|}{\|\theta_0^{\full} \|} \right)^{1/t} = \lim_{t \to \infty} \left( \sup_{\theta_0^{\full} \in \mathbb{R}^{2d}} \frac{\left\|{\left(\mathcal{A}_{\full}^{(\gamma)}\right)}^{t} \theta_0^{\full} \right\|}{\|\theta_0^{\full} \|} \right)^{1/t} = \lim_{t \to \infty} \left\|{\left(\mathcal{A}_{\full}^{(\gamma)}\right)}^{t} \right\|^{1/t},
\end{equation}
which is equal to $\rho\left(\mathcal{A}_{\full}^{(\gamma)}\right)$ by Gelfand's formula.
\end{proof}

\subsection{Empirical distributions of the spectral radii}

Comparing the cyclic, random and full sampling schemes thus requires to compare the values
\begin{equation}\label{eq:alg_spectrum}
    \Aa^\star_{\full} \triangleq \min_{\gamma \in \RR^+} \rho(\Aa_{\full}^{(\gamma)}),\quad 
    \Aa^\star_{\cyc} \triangleq \min_{\gamma \in \RR^+} \rho(\Aa_{\cyc}^{(\gamma)}), \quad
    \Aa^\star_{\rdm} \triangleq \min_{\gamma \in \RR^+} \rho(\Aa_{\rdm}^{(\gamma)}),
\end{equation}
for all matrix games with positive payoff matrix $A \in \RR^{2d \times 2d}$. This is not tractable in closed form. However, we may study the distribution of these values for random games.
\paragraph{Experiment.}We sample matrices $A$ in $\RR^{2 d \times 2d}$ (with $d = 3$) as the weighted sum of a random positive definite matrix $A_{\text{sym}}$ and of a random skew matrix $A_{\text{skew}}$. We refer to \autoref{app:experiments} for a detailed description of the matrix sampling method. We vary the weight $\alpha \in [0, 1]$ of the skew matrix and the lowest eigenvalue $\mu$ of the matrix $A_{\text{sym}}$. We sample $300$ different games and compute $\Aa^{(\eta)}$ on a grid of step sizes $\eta$, for the three different methods. We thus estimate the best algorithmic spectral radii defined in \eqref{eq:alg_spectrum}.

\paragraph{Results and interpretation.} The distributions of algorithm spectral radii are presented in \autoref{fig:radius}. We observe that the algorithm operator associated with sampling one among two players at each update is systematically more contracting than the standard extra-gradient algorithm operator, providing a further insight for the faster rates observed in \autoref{sec:quadratic}, \autoref{fig:quadratic_convergence}. Radius tend to be smaller for cyclic sampling than random sampling, in most problem geometry. This is especially true in well conditioned problem (high $\mu$), little-skew problems (skewness $\alpha < .5$) and completely skew problems $\alpha = 1$. The later gives insights to explain the good performance of cyclic player sampling for GANs (\autoref{sec:gans}), as those are described by skew games (zero-sum notwithstanding the discriminator penalty in WP-GAN).

On the other hand, we observe that radii are more spread using cyclic sampling for intermediary skew problerm ($\alpha = .75$), hinting that worst-case rates may be better for random sampling.

\begin{figure}
    \centering
    \includegraphics[width=\textwidth]{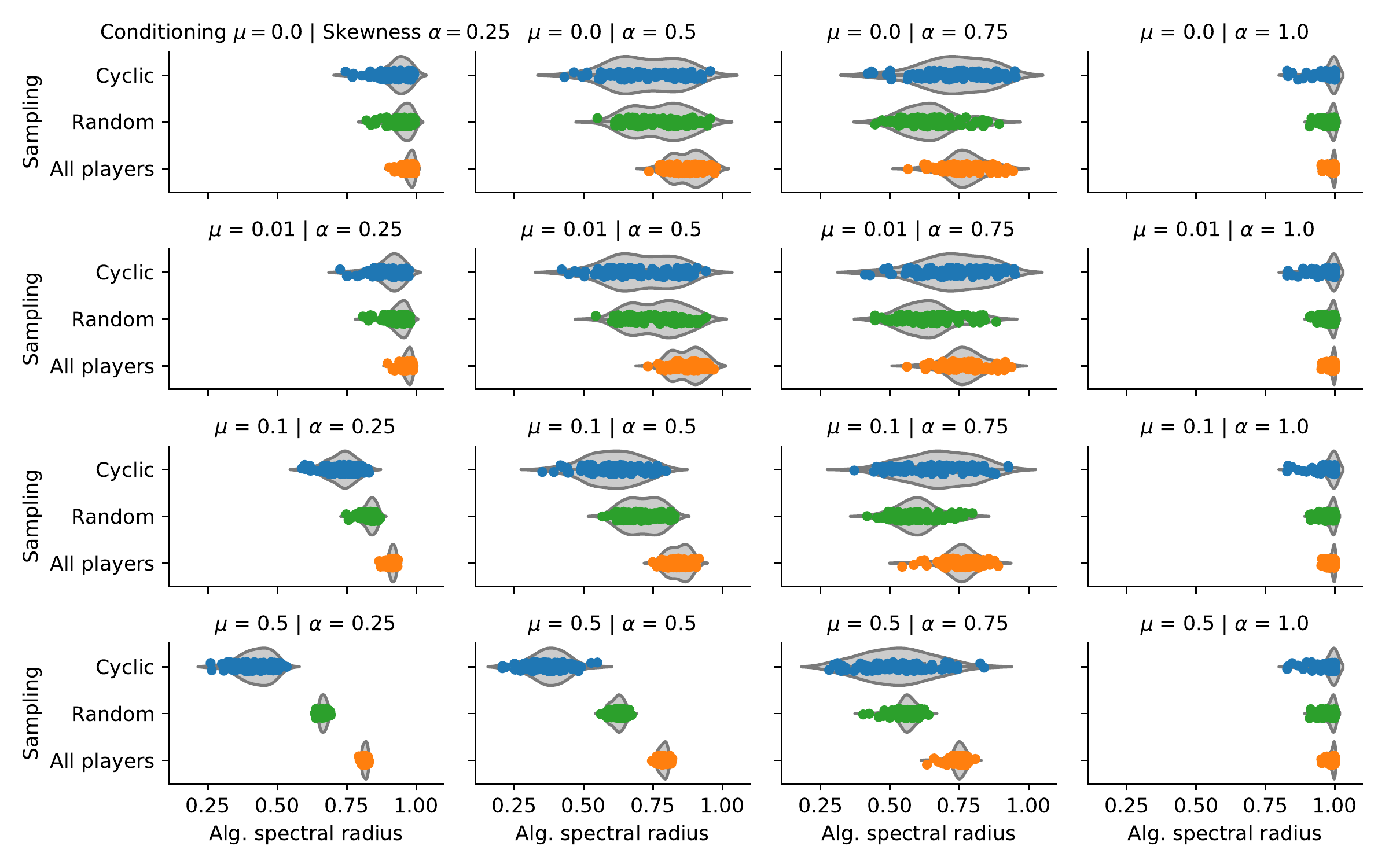}
    \caption{Spectral radii distribution of the algorithmic operator associated to doubly-stochastic and full extra-gradient, in the non-constrained bi-linear two-player game setting, for various conditioning and skewness. Random and cyclic sampling yields lower radius (hence faster rates) for most problem geometry. Cyclic sampling outperforms random sampling in most settings, especially for better conditioned problems.}\label{fig:radius}
\end{figure}



%% file: appendix_exp.tex

\clearpage
\section{Experimental results and details}\label{app:experiments}

We provide the necessary details for reproducing the  experiments of \autoref{sec:apps}.

\subsection{Quadratic games}

\paragraph{Generation of random matrices.}

We sample two random Gaussian matrix $G$ and $F$ in $\RR^{nd \times nd}$, where each coefficient $g_{ij}, f_{ij} \sim \Nn(0, 1)$ is sampled independently.
We form a symmetric matrix $A_{\text{sym}} = \frac{1}{2} (G + G^T)$, and a skew matrix $A_{\text{skew}} = \frac{1}{2} (F - F^T)$.
 To make $A_{\text{sym}}$ positive definite, we compute its lowest eigenvalue $\mu_0$, and update $A_{\text{sym}} \gets A_{\text{sym}} + (\mu - \mu_0) I_{nd \times nd}$, where $\mu$ regulates the conditioning of the problem and is set to $0.01$.
 We then form the final matrix $A = (1- \alpha) A_{\text{sym}} + \alpha A_{\text{skew}}$, where $\alpha$ is a parameter between $0$ and $1$, that regulates the skewness of the game.

 \begin{figure}[t]
    \centering
    \includegraphics[width=\textwidth]{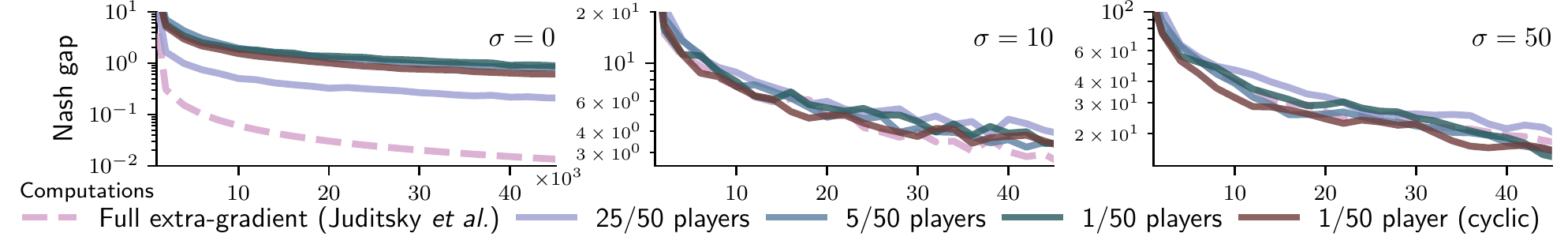}
    \caption{50-player completely skew smooth game with increasing noise (sampling with variance reduction). In the non-noisy setting, player sampling reduces convergence speed. On the other hand, it provides a speed-up in the high noise regime.}\label{fig:skew}
\end{figure}

\paragraph{Parameters for quadratic games.}\autoref{fig:quadratic_convergence} compare rates of convergence for doubly-stochastic extra-gradient and extra-gradient, for increasing problem complexity. Used parameters are reported in \autoref{table:quad_parameters}. Note that the conclusion reported in \autoref{sec:quadratic} regarding the impact of noise and the impact of cyclic sampling holds for all configurations we have tested; we designed increasingly complex experiments for concisely showing the efficiency and limitations of doubly-stochastic extra-gradient.

\begin{table}[b]
    \caption{Parameters used in \autoref{fig:quadratic_convergence} for increasing problem complexity. \label{table:quad_parameters}}
    \centering
    \begin{tabular}{lclccc}
        \toprule
        Figure & Players \# & Exp. & Skewness $\alpha$ & Noise $\sigma$ & Reg. $\lambda$ \\ 
        \midrule
        \autoref{fig:quadratic_convergence_5}    & 5 & Smooth, no-noise & $0.9$ & $0$ & $0$ \\
        &  & Smooth, noisy & $0.9$ & $1$ & 0. \\
        &  & Skew, non-smooth, noisy & $1.$ & $1$ & $2\cdot 10^2$ \\
        \midrule
        \autoref{fig:quadratic_convergence_50}    & 50 & Smooth, no-noise & $0.9$ & $0$ & $0$ \\
        &  & Non-smooth, noisy & $0.9$ & $1$ & $2\cdot 10^{-2}$ \\
        &  & Skew, non-smooth, noisy & $1.$ & $1$ & $2\cdot 10^{-2}$ \\
        \midrule
        \autoref{fig:quadratic_convergence_noise} &  50 & Smooth, skew, lowest-noise  & $0.95$ & $1$ & $0.$ \\
        &  &  & $0.95$ & $10$ & $0.$ \\
        &  & Smooth, skew, highest-noise & $0.95$ & $100$ & $0.$ \\
        \midrule
        \autoref{fig:skew} &  50 & Smooth, skew, no-noise  & $1$ & $0$ & $0.$ \\
        &  &  & $1$ & $10$ & $0.$ \\
        &  & Smooth, skew, highest-noise & $1$ & $50$ & $0$ \\
        \midrule
    \end{tabular}
\end{table}

\paragraph{Grids.}For each experiment, we sampled $5$ matrices $(A_i)_i$ with skewness parameter~$\alpha$. We performed a grid-search on learning rates, setting $\eta \in \{10^{-5}, \cdots, 1\}$, with $32$ logarithmically-spaced values, making sure that the best performing learning rate is always strictly in the tested range.

\paragraph{Limitations in skew non-noisy games.}As mentioned in the main section, player sampling can hinder performance in completely skew games ($\alpha = 1$) with non-noisy losses. Those problems are the hardest and slower to solve. They corresponds to \textit{fully adversarial} settings, where sub-game between each pair is zero-sum. We illustrate this finding in \autoref{fig:skew}, showing how the performance of player sampling improves with noise. We emphasize that the non-noisy setting is not relevant to machine learning or reinforcement learning problems.

\subsection{Generative adversarial networks}

\paragraph{Models and loss.} We use the Residual network architecture for generator and discriminator proposed by~\citet{gidel2018variational}. We use a WGAN-GP loss, with gradient penalty $\lambda = 10$. As advocated by \cite{gidel2018variational}, we use a 10 times lower stepsize for the generator. We train the generator and discriminator using the Adam algorithm~\citep{kingma_adam_2014}, and its straight-forward extension proposed by~\cite{gidel2018variational}.

\paragraph{Grids.}We perform $5 \cdot 10^5$ generator updates. We average each experiments with $5$ random seeds, and select the best performing generator learning rate $\eta \in \{2 \cdot 10^{-5},5 \cdot 10^{-5}, 8 \cdot 10^{-5}, 1 \cdot 10^{-4}, 2 \cdot 10^{-4}\}$, which turned out to be $5 \cdot 10^{-5}$ for both subsampled and non-subsampled extra-gradient.

